\documentclass[runningheads]{llncs}

\usepackage{eccv}

\usepackage{eccvabbrv}

\usepackage{graphicx}
\usepackage{booktabs}

\usepackage[accsupp]{axessibility}  

\usepackage{hyperref}

\usepackage{orcidlink}

\usepackage{booktabs}       
\usepackage{amsfonts}       
\usepackage{caption}
\usepackage{graphicx}
\usepackage{multirow}
\usepackage{rotating}
\usepackage{float}
\usepackage{breqn}
\usepackage{algorithm}
\usepackage{algorithmic}
\usepackage{wrapfig}
\usepackage{mathtools}
\usepackage{pifont}
\usepackage{listings}
\usepackage{bm}
\usepackage[T1]{fontenc}
\usepackage{xcolor}
\definecolor{mygreen}{HTML}{4FAD5B}
\usepackage{colortbl}
\definecolor{SkyBlue}{RGB}{71, 189, 234}

\begin{document}

\title{\includegraphics[width=0.7cm]{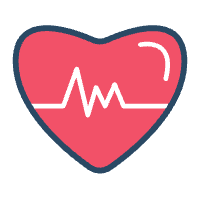}\textcolor{SkyBlue}{FreqPhys:} Repurposing Implicit Physiological Frequency Prior for Robust Remote Photoplethysmography}

\titlerunning{FreqPhys: Repurposing Implicit Physiological Frequency Prior for rPPG}

\author{Wei Qian\inst{1} \and
Dan Guo\inst{1,2} \and
Jinxing Zhou\inst{3} \and
Bochao Zou\inst{4} \and
Zitong Yu\inst{5} \and \\
Meng Wang\inst{1,2}}

\authorrunning{W. Qian et al.}

\institute{Hefei University of Technology \and
Hefei Comprehensive National Science Center \and
MBZUAI \and
University of Science and Technology Beijing \and
Great Bay University
}

\maketitle

\begin{abstract}
Remote photoplethysmography (rPPG) enables contactless physiological monitoring by capturing subtle skin-color variations from facial videos. However, most existing methods predominantly rely on time-domain modeling, making them vulnerable to motion artifacts and illumination fluctuations, where weak physiological clues are easily overwhelmed by noise. To address these challenges, we propose \textcolor{SkyBlue}{\textbf{FreqPhys}}, a frequency-guided rPPG framework that explicitly leverages physiological frequency priors for robust signal recovery. Specifically, \textcolor{SkyBlue}{\textbf{FreqPhys}} first applies a \textit{Physiological Bandpass Filtering} module to suppress out-of-band interference, and then performs \textit{Physiological Spectrum Modulation} together with \textit{adaptive spectral selection} to emphasize pulse-related frequency components while suppress residual in-band noise. A \textit{Cross-domain Representation Learning} module further fuses these spectral priors with deep time-domain features to capture informative spatial--temporal dependencies. Finally, a frequency-aware conditional diffusion process progressively reconstructs high-fidelity rPPG signals. Extensive experiments on six benchmarks demonstrate that \textcolor{SkyBlue}{\textbf{FreqPhys}} yields significant improvements over state-of-the-art approaches, particularly under challenging motion conditions. It highlights the importance of explicitly modeling physiological frequency priors. 
The source code will be released.
\end{abstract}

\section{Introduction}
Physiological signals such as heart rate (HR), heart rate variability (HRV), and respiration frequency (RF) are fundamental indicators of physical and mental health.
Traditional electrocardiogram (ECG) and photoplethysmography (PPG) systems require direct skin contact, which may cause discomfort and limit their applicability.
Remote photoplethysmography (rPPG)~\cite{verkruysse2008remote} has therefore emerged as a promising non-invasive alternative optical technique, enabling applications in health monitoring~\cite{huang2023challenges}, face anti-spoofing~\cite{yu2021transrppg}, and psychological stress assessment~\cite{gedam2020automatic}, among others.
Despite its potential, reliably extracting rPPG signals, subtle skin color changes caused by blood volume fluctuations in facial videos captured by commodity cameras, remains challenging, particularly under motion and illumination variation.

\begin{figure*}[t]
\centering
\includegraphics[width=0.9\linewidth]{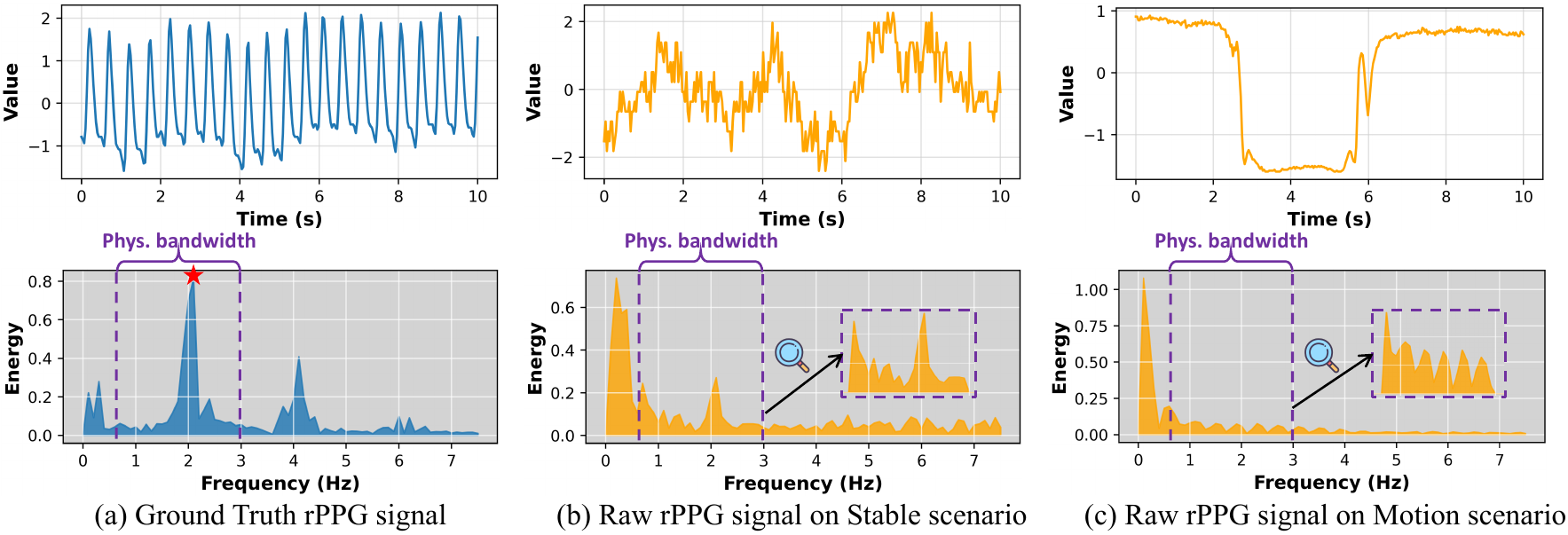}
\caption{\textbf{Visualization of the differences between ground-truth and raw rPPG signals in time and frequency domains.}
(a) Ground-truth rPPG signal, whose spectrum exhibits clear physiological priors: (\textcolor{magenta}{\emph{i}}) \textbf{Physiological Band Constraint}, where the spectral energy is concentrated within the physiological bandwidth of [0.66, 3.0] Hz corresponding to normal heart rate ranges, and (\textcolor{magenta}{\emph{ii}}) \textbf{Dominant Peak Property}, where a prominent spectral peak within this band reflects the periodic cardiac rhythm, while other in-band components remain comparatively low. The dominant frequency peak (marked by {\color{red}\ding{72}}) corresponds to heart rate and is converted to beats per minute by multiplying the frequency by 60.
(b)(c) Raw rPPG signals extracted from facial videos under stable and motion conditions, computed by averaging green-channel pixel intensities over time~\cite{wang2016algorithmic}.
In the time domain, physiological signals are heavily entangled with noise. 
In the frequency domain, noise manifests as both out-of-band interference and residual in-band components, predominantly concentrated at lower frequencies. 
While the stable scenario preserves a relatively distinct spectral peak, motion artifacts disperse the in-band energy distribution, making reliable denoising considerably more challenging.}
\label{fig:motivation}
\end{figure*}

Early rPPG studies~\cite{verkruysse2008remote,poh2010non,de2013robust,wang2016algorithmic} mainly relied on handcrafted signal processing pipelines, which depended heavily on specific assumptions.
With the rapid development of deep learning, numerous data-driven rPPG models have been proposed~\cite{yu2019remote1,niu2020video,lu2021dual,liu2023efficientphys,qian2024dual,zou2025rhythmmamba}. 
Although these methods achieve strong performance under stable conditions, their performance often degrades in complex scenarios due to noise sources such as motion artifacts and illumination variations~\cite{qian2025physdiff,shao2025remote}. 
More recently, diffusion models have been introduced for rPPG estimation~\cite{chen2024diffphys,qian2025physdiff}, leveraging their powerful capability to model complex noise distributions and progressively recover clean signals. 
Nevertheless, these approaches predominantly operate in the time domain, where noise appears in irregular and highly entangled forms (Fig.~\ref{fig:motivation}(b)(c)), making it inherently challenging to disentangle physiological components from interference.
Several recent studies have attempted to incorporate frequency-domain information into deep rPPG models. 
Typical strategies include introducing auxiliary spectral losses~\cite{yu2023physformer++,sun2024contrast,zou2025rhythmmamba}, embedding Fourier-based modules to enrich feature representations~\cite{zou2025rhythmmamba}, or synthesize frequency-modulated augmentations~\cite{yue2023facial}. 
While these efforts demonstrate the potential value of spectral clues, they do not explicitly leverage the intrinsic physiological priors that characterize cardiac-driven rPPG signals, nor do they fully address the structured spectral distortions introduced by motion artifacts.

In this work, we revisit the frequency domain and show that rPPG signals, driven by quasi-periodic cardiac rhythms, exhibit strong and well-defined spectral priors~\cite{gideon2021way,speth2023non}, as illustrated in Fig.~\ref{fig:motivation}(a): (\textcolor{magenta}{\emph{i}}) \textbf{Physiological Band Constraint:} the spectral energy is largely concentrated within a physiological frequency band, typically [0.66, 3.0] Hz, corresponding to the normal heart-rate range; and (\textcolor{magenta}{\emph{ii}}) \textbf{Dominant Peak Property:} a prominent spectral peak emerges within this band, reflecting the periodic cardiac rhythm, while other in-band components remain comparatively weak. 
To further illustrate this phenomenon, Fig.~\ref{fig:motivation}(b)(c) presents raw rPPG signals extracted under stable and motion conditions by computing the mean green-channel intensity over time~\cite{wang2016algorithmic}. 
While time-domain signals exhibit severe entanglement between physiological patterns and noise, their frequency-domain representations reveal two distinct types of interference: (\textcolor{magenta}{\emph{i}}) \textbf{out-of-band components dominated by low-frequency drift}, and (\textcolor{magenta}{\emph{ii}}) \textbf{residual in-band disturbances that blur the cardiac peak under motion}. 
These observations naturally raise a central question: 
\textbf{\textit{How can we simultaneously suppress out-of-band and in-band noise while preserving physiologically meaningful spectral structures?}}

To address this challenge, we propose \textcolor{SkyBlue}{\textbf{FreqPhys}}, a diffusion-based rPPG framework that repurposes implicit physiological frequency priors to guide robust signal recovery. 
Specifically, to remove out-of-band interference, we first apply a \textit{Physiological Bandpass Filter} that retains only cardiac-relevant spectral components. 
To further emphasize true cardiac harmonics and mitigate residual in-band noise, we introduce \textit{Physiological Spectrum Modulation} together with \textit{Adaptive Spectral Selection}, enabling dynamic enhancement and suppression within the physiological frequency range. 
Finally, we design a \textit{Cross-domain Representation Learning} module that employs cross-attention to fuse spectral priors with deep time-domain representations, providing the diffusion model with explicit physiological regularities and informative temporal dependencies.
Our main contributions are summarized as follows:
\begin{itemize}
    \item We highlight the importance of explicitly leveraging implicit physiological frequency priors for robust rPPG estimation, revealing the structured spectral characteristics induced by cardiac rhythms.
    \item We propose \textcolor{SkyBlue}{\textbf{FreqPhys}}, a frequency-aware diffusion framework that integrates physiological frequency denoising with cross-domain representation learning for robust rPPG signal reconstruction.
    \item Unlike prior diffusion-based rPPG methods that operate solely in the time domain, our method incorporates frequency-domain conditioning to better capture the quasi-periodic structure of rPPG signals.
    \item Extensive experiments on six public datasets demonstrate that our method achieves superior performance in both robustness and generalization, especially under challenging motion conditions.
\end{itemize}

\section{Related Work}
\textbf{Traditional rPPG Methods.} Early rPPG studies primarily relied on classical signal processing techniques, including spatial averaging or color–separation–based approaches such as GREEN~\cite{verkruysse2008remote}, ICA~\cite{poh2010non}, CHROM~\cite{de2013robust}, and POS~\cite{wang2016algorithmic}.
These methods are computationally efficient and easy to implement; however, they typically depend on strong assumptions (e.g., stable illumination conditions or linear color reflection models). 
As a result, their performance degrades considerably in unconstrained complex scenarios where motion artifacts and illumination variations are prevalent.

\noindent\textbf{Deep Learning–based rPPG Methods.} The success of deep learning has led to a wide range of data-driven rPPG models. 
CNN-based architectures, such as DeepPhys~\cite{chen2018deepphys}, PhysNet~\cite{yu2019remote1}, CVD~\cite{niu2020video}, and TS-CAN~\cite{liu2020multi}, focus on learning local spatial–temporal representations from facial videos. 
To capture longer temporal dependencies, Transformer-based models, including PhysFormer~\cite{yu2022physformer}, EfficientPhys~\cite{liu2023efficientphys}, and Dual-TL~\cite{qian2024dual}, have been proposed. 
More recently, state-space and Mamba-style architectures (e.g., PhysMamba~\cite{luo2024physmamba} and RhythmMamba~\cite{zou2025rhythmmamba}) have been explored to enlarge the temporal receptive field while maintaining efficient sequential modeling. 
In addition, diffusion-based frameworks, such as DiffPhys~\cite{chen2024diffphys} and PhysDiff~\cite{qian2025physdiff}, leverage generative denoising processes to progressively recover clean physiological signals from noisy temporal input.

\noindent\textbf{Frequency-domain rPPG Methods.} Beyond purely time-domain modeling, several studies have explored incorporating spectral clues into rPPG estimation. 
Existing approaches generally fall into three categories: 
(\emph{\textcolor{magenta}{i}}) introducing frequency-based loss functions to regularize training, as in PhysFormer++~\cite{yu2023physformer++}, ContrastPhys~\cite{sun2024contrast}, and RhythmMamba~\cite{zou2025rhythmmamba}; 
(\emph{\textcolor{magenta}{ii}}) designing frequency-aware representations, such as Fourier-based blocks or spectral embeddings~\cite{zou2025rhythmmamba}; and 
(\emph{\textcolor{magenta}{iii}}) leveraging frequency-guided augmentation or synthesis strategies, for example, generating negative samples through spectral manipulation~\cite{yue2023facial}. 
However, these approaches mainly incorporate spectral information as auxiliary supervision or representation enhancement, without explicitly addressing the structured spectral distortions introduced by motion and illumination. In practice, these disturbances manifest as both out-of-band interference and residual in-band noise that overlap with cardiac rhythms, making it difficult for existing methods to reliably isolate physiological components.
A more comprehensive discussion of related work is provided in \textbf{Appendix}~\ref{appendix_related_work}.

\noindent\textbf{Remark}. Our work revisits the role of frequency information in rPPG modeling.
Existing approaches typically use spectral cues either (\emph{\textcolor{magenta}{i}}) during post-processing for HR estimation~\cite{niu2020video,lu2021dual,yu2022physformer,liu2023efficientphys,qian2024dual}, or (\emph{\textcolor{magenta}{ii}}) as auxiliary training components, such as Fourier blocks~\cite{zou2025rhythmmamba}, spectral augmentations~\cite{yue2023facial}, or frequency-based losses~\cite{yu2023physformer++,sun2024contrast}. 
In contrast, \textcolor{SkyBlue}{\textbf{FreqPhys}} incorporates implicit physiological frequency priors directly into the online denoising process as an internal structural constraint, rather than treating them as external or optional clues. 
Specifically, we introduce a progressive three-stage frequency-guided diffusion denoising framework consisting of (\emph{\textcolor{magenta}{i}}) out-of-band suppression, (\emph{\textcolor{magenta}{ii}}) cardiac harmonic enhancement, and (\emph{\textcolor{magenta}{iii}}) adaptive in-band selection, which is applied at every diffusion step during both training and inference.

\begin{figure*}[t]
\centering
\includegraphics[width=1.0\linewidth]{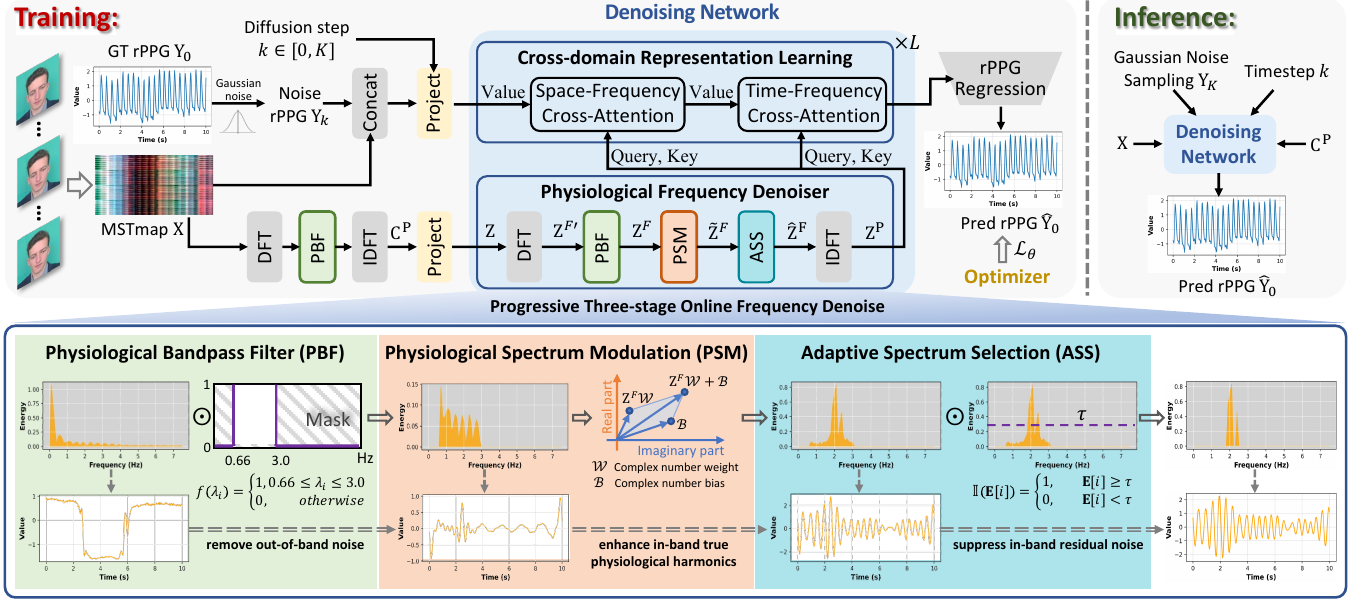}
\caption{\textbf{The pipeline of proposed \textcolor{SkyBlue}{\textbf{FreqPhys}}.}
Given a facial video, we first construct MSTmap $\mathbf{X}$ as the temporal condition and generate the frequency condition $\mathbf{C}^{\mathbf{P}}$ by applying the PBF.
During training, we initially generate noise rPPG $\mathbf{Y}_{k}$ by adding Gaussian noise to Ground Truth rPPG $\mathbf{Y}_{0}$ for the $k$-th step.
Then, we input $\mathbf{Y}_{k}$, $\mathbf{X}$, $k$, and $\mathbf{C}^{\mathbf{P}}$ into the \textit{Denoising Network}.
Specifically, the frequency condition $\mathbf{C}^{\mathbf{P}}$ is fed into the \textit{Physiological Frequency Denoiser} module to enhance physiological spectral clues through three key steps: (\emph{\textcolor{magenta}{i}}) PBF removes out-of-band noise based on the physiological frequency bandwidth [0.66,3.0] Hz; (\emph{\textcolor{magenta}{ii}}) PSM emphasizes valid physiological harmonics by modeling interactions between real and imaginary components;  (\emph{\textcolor{magenta}{iii}}) ASS dynamically suppresses in-band noise using data-driven energy thresholds.
Next, with \textit{Cross-domain Representation Learning}, our \textcolor{SkyBlue}{\textbf{FreqPhys}} includes frequency-domain denoised information into space and time dependencies modeling to estimate the high-fidelity rPPG signal.
During inference, the initial rPPG $\mathbf{Y}_{K}$ is randomly sampled from Gaussian noise, with frequency condition and denoising network processes mirroring those used in training.
}
\label{fig:pipeline}
\end{figure*}

\section{Methodology}
Remote physiological measurement from facial videos can be regarded as a video sequence to signal sequence problem. 
Let $\mathbf{V} \in \mathbb{R}^{T \times 3 \times H \times W}$ denote a raw facial video clip containing $T$ frames with 3 color channels and spatial resolution $H \times W$.
Following established rPPG preprocessing protocols ~\cite{niu2020video,qian2025physdiff}, we extract $N$ facial regions of interest (ROI) through landmark alignment and pixel-level average pooling of $C$ color channels, constructing a multi-scale temporal map (MSTmap) $\mathbf{X} \in \mathbb{R}^{T \times N \times C}$ as the model input. 
The objective is to recover the clean periodic rPPG signal $\mathbf{Y} \in \mathbb{R}^T$ from $\mathbf{X}$, formulated as learning a denoising function $f_\theta: \mathbf{X} \mapsto \mathbf{Y}$, where $\theta$ denotes trainable parameters. 
The overview of our method is illustrated in Fig.~\ref{fig:pipeline}, where the details are described as follows.

\subsection{Physiological Frequency Denoiser}
\textbf{Physiological Bandpass Filter (PBF). }
Inspired by the fact that true cardiac activities mainly fall within a fixed frequency bandwidth, typically [0.66,3.0] Hz~\cite{wang2016algorithmic}, we devise a \textit{Physiological Bandpass Filter} that directly isolates cardiac frequency components in the spectral space.
Specifically, we first project the frequency condition $\mathbf{C}^{\mathbf{P}}\in \mathbb{R}^{T \times N \times C}$ into a $D$-dimensional latent space $\mathbf{Z}\in \mathbb{R}^{T \times N \times D}$ by a linear layer, then transform it to frequency domain via the Discrete Fourier Transform $\mathcal{F}$,
\begin{equation}
    \mathbf{Z}^{F'} = \mathcal{F} (\mathbf{Z}) \in \mathbb{C}^{(\lfloor T/2 \rfloor+1) \times N\times D}.
\end{equation}
Then, the noisy frequency components outside the physiological bandwidth range can be discarded using an ideal bandpass filter:
\begin{equation}
    {\mathbf{Z}}^{F} = \mathbf{Z}^{F'}\odot f(\lambda_i), \quad \text{for } i = 0, \dots, \lfloor T/2\rfloor,
\end{equation}
where $\odot$ denotes the Hadamard product. $\lambda_i$ denotes the physical frequency corresponding to the $i$-th frequency bins and $\lambda_i = if_s/T$ Hz, with $f_s$ denoting the sampling rate. 
$f(\lambda_i)$ is an indicator function, which outputs 1 when $(0.66\leq\lambda_i\leq\;3.0)$ and 0 otherwise.

\begin{wrapfigure}[11]{r}{0.45\linewidth}
\vspace{-2.5ex}
\centering
\includegraphics[width=1.0\linewidth]{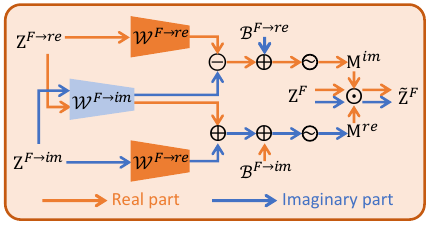}
\vspace{-4ex}
\caption{\textbf{The details of physiological spectrum modulation module}.} 
\label{fig:PSM}
\end{wrapfigure}
\noindent\textbf{Physiological Spectrum Modulation (PSM). } 
While the \textit{Physiological Bandpass Filter} is effective in removing noise outside the physiological frequency band, another challenge still exists where noise components may overlap or closely resemble physiological signals within this band. 
Such overlapping frequencies may severely distort the signal, making it difficult to accurately extract physiological features.
We apply a learnable \textit{Physiological Spectrum Modulation} in the frequency domain to emphasize true physiological harmonics while suppressing non-physiological components.
The detailed implementation process of PSM is shown in Fig.~\ref{fig:PSM}.
Specifically, given the physiological frequency representation $\mathbf{Z}^{F}=\mathbf{Z}^{F \rightarrow re}+j\cdot\mathbf{Z}^{F \rightarrow im} \in\mathbb{C}^{(\lfloor T/2 \rfloor+1)\times N\times D}$, we denote its real and imaginary parts as $\mathbf{Z}^{F \rightarrow re}$ and $\mathbf{Z}^{F \rightarrow im}$, separately.
To achieve more exhaustive spectrum modulation, we encode the real and imaginary parts separately to generate the modulation signals:
\begin{equation}
    \mathbf{M} = \mathbf{M}^{re} + j \cdot\mathbf{M}^{im} = \sigma(\mathbf{Z}^{F} \mathcal{W}^{F} + \mathcal{B}^{F}), \mathbf{M}\in\mathbb{C}^{(\lfloor T/2 \rfloor+1)\times N\times D},
\end{equation}
where $\sigma$ is the ReLU activation function, $\mathcal{W}^{F} = (\mathcal{W}^{F \rightarrow re} + j \cdot \mathcal{W}^{F \rightarrow im}) \in\mathbb{C}^{D\times D}$ is the trainable complex number weight matrix with $\{\mathcal{W}^{F \rightarrow re}, \mathcal{W}^{F \rightarrow im}\} \in\mathbb{R}^{D\times D}$, and $\mathcal{B}^{F} = (\mathcal{B}^{F \rightarrow re} + j \cdot \mathcal{B}^{F \rightarrow im}) \in\mathbb{C}^{D}$ is the trainable complex number biases with $\{\mathcal{B}^{F \rightarrow re},\mathcal{B}^{F \rightarrow im}\} \in\mathbb{R}^{D}$. 
According to the rule of multiplication of complex numbers (details can be seen in \textbf{Appendix}~\ref{complex_multiplication}), we further unfold into real and imaginary parts as follows:
\begin{equation}
    \begin{aligned}
        &\mathbf{M}^{re} = \sigma(\mathbf{Z}^{F \rightarrow re} \mathcal{W}^{F \rightarrow re} - \mathbf{Z}^{F \rightarrow im} \mathcal{W}^{F \rightarrow im} + \mathcal{B}^{F \rightarrow re}), \\
        &\mathbf{M}^{im} = \sigma(\mathbf{Z}^{F \rightarrow re} \mathcal{W}^{F \rightarrow im} + \mathbf{Z}^{F \rightarrow im} \mathcal{W}^{F \rightarrow re}+ \mathcal{B}^{F \rightarrow im}). \\
    \end{aligned}
\end{equation}
Afterwards, the generated complex signal is used to modulate counterparts of the original frequency-domain feature, which can be written as,
\begin{equation}\label{eq:circular_convolution}
    \mathbf{\tilde{Z}}^{F} = \mathbf{M}\odot\mathbf{Z}^{F} \in \mathbb{C}^{(\lfloor T/2\rfloor+1)\times D}.
\end{equation}
By Theorem~\ref{theorem_1}, this spectral multiplication operation in Equation~\ref{eq:circular_convolution} is mathematically equivalent to a global circular convolution in time, endowing each sequence with a content-adaptive receptive field that is ideal for capturing periodic cardiac rhythms.  
By means of Equation~\ref{eq:circular_convolution}, the physiological frequency components can be effectively enhanced via direct spectral modulation.
The detailed proof of \textit{theorem}~\ref{theorem_1} is illustrated in \textbf{Appendix}~\ref{theorem_1_proof}.
\begin{theorem}\label{theorem_1}
    (Frequency-domain Convolution Theorem)  The multiplication of two signals in the frequency domain is equivalent to the frequency transformation of a circular convolution of these two signals in the temporal domain, which can be summarized as:
    \begin{equation}
        \mathcal{F} [\mathbf{M}(v)\otimes\mathbf{Z}(v)] = \mathcal{F}(\mathbf{M}(v))\odot\mathcal{F}(\mathbf{Z}(v)),
    \end{equation}
    where $\otimes$ and $\odot$ represent circular convolutional operation and element multiplication operation, respectively, $\mathbf{M}(v)$ and $\mathbf{Z}(v)$ represent two signals for the time variable $v$, and $\mathcal{F}(\cdot)$ denotes the Discrete Fourier Transform.
\end{theorem}

\noindent\textbf{Adaptive Spectrum Selection (ASS). }  
Although the physiological spectrum modulation block enhances the cardiac frequency band, residual noise components whose frequencies lie close to the physiological range may still persist. 
To further isolate the true pulse periodicity, we introduce an \emph{Adaptive Spectrum Selection} (ASS) module that performs learnable frequency-domain filtering via a data-driven threshold. 
Given the modulated spectrum $\widetilde{\mathbf{Z}}^F$, we first compute its per-frequency energy:
\begin{equation}
\begin{split}
    \mathbf{E}[i] = \sqrt{(\widetilde{\mathbf{Z}}^{F \rightarrow re}[i])^2
    + (\widetilde{\mathbf{Z}}^{F \rightarrow im}[i])^2} 
    = \left\|\widetilde{\mathbf{Z}}^F[i]\right\|_2, 
    \quad i = 0, 1, \dots, \lfloor T/2 \rfloor.
\end{split}
\end{equation}
We then introduce a learnable threshold $\tau$ to adaptively distinguish dominant cardiac frequencies from residual noise. 
A binary selection mask is defined as:
\begin{equation}
m[i] = \mathbb{I}(\mathbf{E}[i] \ge \tau),
\end{equation}
where $\mathbb{I}(\cdot)$ denotes the indicator function. 
The filtered spectrum is computed as:
\begin{equation}
\hat{\mathbf{Z}}^F[i] = \widetilde{\mathbf{Z}}^F[i] \odot m[i].
\end{equation}
Importantly, the hard selection mask is trained using the Straight-Through Estimator (STE) strategy described previously, allowing the adaptive threshold $\tau$ to be jointly optimized with the entire network.
Finally, the selected spectrum is transformed back to the time domain via inverse DFT:
\begin{equation}
\mathbf{Z}^{\mathbf{P}} = \mathcal{F}^{-1}(\hat{\mathbf{Z}}^F),
\end{equation}
yielding a pulse signal dominated by cardiac frequency components.

\subsection{Cross-domain Representation Learning}
Time domain modeling focuses on local dependencies and transient behaviors, while frequency domain analysis provides insights into the global correlations and periodicity of the data. 
Therefore, combining these two domains is a promising approach to recover high-fidelity rPPG signals.
To effectively integrate intermediate representations $\mathbf{Z}$ in the time domain with frequency-domain priors $\mathbf{Z}^{\mathbf{P}}$, we propose a cross-domain representation learning module.
Specifically, we perform $L$ alternating cross-attention layers that can progressively learn the various input domains for representation learning.
In each layer $l$, the initial physiological frequency representation is first obtained by applying PBF, followed by PSM and ASS modules.
Next, the physiological frequency representation $\mathbf{Z}^{\mathbf{P}}$ and intermediate representations $\mathbf{Z}$ are integrated by cross-attention across the spatial and temporal axis, respectively. 
Formally, this process can be formulated as follows:
\begin{equation}
    \begin{aligned}
        &\mathbf{Z}^{\mathbf{P},(l)}= \mathrm{ASS}(\mathrm{PSM}(\mathrm{PBF}(\mathbf{Z}^{(l)}))),\\
        &\mathbf{Z}^{(l)'}= \mathrm{CA}(\mathbf{Z}^{\mathbf{P},(l)}, \mathbf{Z}^{(l)})+\mathbf{Z}^{(l)},\\
        &\mathbf{Z}^{(l+1)}= \mathrm{CA}(\mathbf{Z}^{\mathbf{P},(l)}, \mathbf{Z}^{(l)'})+\mathbf{Z}^{(l)'},\\
    \end{aligned}
\end{equation}
where $\mathrm{CA}(a,b)$ refers to Cross-Attention, with $a$ denotes query and key, and $b$ denotes value.

\begin{figure*}[t]
\centering
\includegraphics[width=1.0\linewidth]{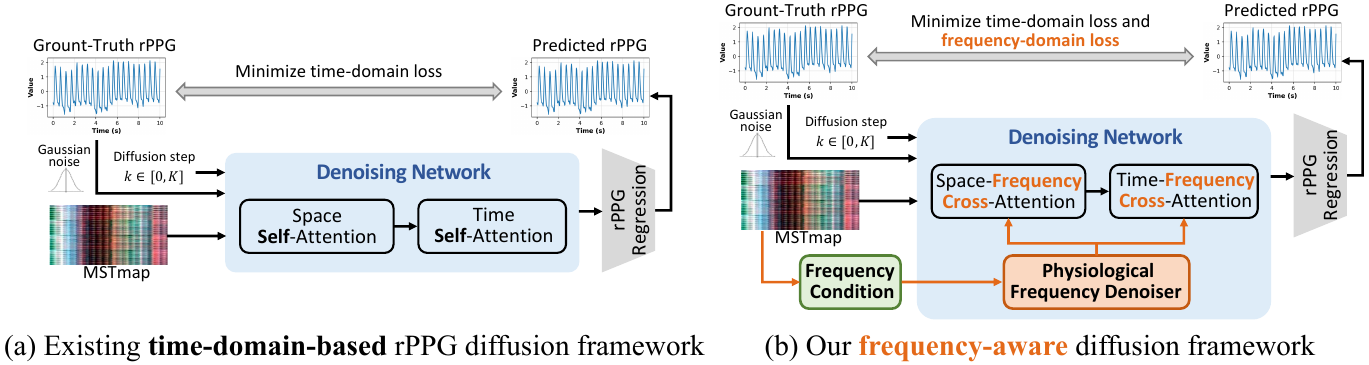}
\caption{\textbf{Architecture comparison with existing diffusion methods}.} 
\label{fig:diffusion_arch}
\end{figure*}

\subsection{Frequency‐aware Diffusion Model}\label{sec:3.3}
Recently, denoising diffusion probabilistic models (DDPMs)~\cite{ho2020denoising} have emerged as powerful generative frameworks that progressively refine noisy inputs through learned reverse Markov chains, capturing complex data distributions.
Inspired by this, some diffusion-based rPPG methods~\cite{qian2025physdiff,chen2024diffphys} for rPPG estimation have been proposed and achieved SOTA performance.
They treat the rPPG estimation task as calculating the conditional rPPG signal probability distribution $q(\mathbf{Y}_{0}|\mathbf{C})$, where $q(\mathbf{Y}_{0})$ is the clean rPPG distribution, and the condition $\mathbf{C}$ for probability distribution calculation is generally the input $\mathbf{X}$ in the time domain.
However, these diffusion models mainly focus on time-domain conditioning and overlook the unique spectral prior. 
To alleviate this limitation, we introduce a novel frequency-aware diffusion model that explicitly incorporates physiological frequency priors to guide the generation of high-fidelity rPPG signals, as shown in Fig.~\ref{fig:diffusion_arch}. 
Specifically, our frequency-aware diffusion model fuses with physiological frequency condition $\mathbf{C}^{\mathbf{P}}$ to learn the conditional rPPG distribution $q(\mathbf{Y}_0 | \mathbf{Y}, \mathbf{X}, \mathbf{C}^{\mathbf{P}})$, through two Markov chain processes of diffusion step $K$, i.e., the forward process and the reverse process.

\noindent\textbf{Forward Process.}
The forward process $q$ incrementally adds Gaussian noise to the ground truth rPPG signal $\mathbf{Y}_0\in \mathbb{R}^{T}$ via a fixed Markov chain $\mathbf{Y}_0,\dots \mathbf{Y}_K$ as follows:
\begin{equation}\label{eq:process_q}
\begin{split}
    q(\mathbf{Y}_{1:K} | \mathbf{Y}_0) &= \prod_{k=1}^K q(\mathbf{Y}_k | \mathbf{Y}_{k-1}), \\
    q(\mathbf{Y}_k | \mathbf{Y}_{k-1}) &= \mathcal{N}(\mathbf{Y}_{k}; \sqrt{1-\beta_k}\mathbf{Y}_{k-1}, \beta_k \mathbf{I}),
\end{split}
\end{equation}
where $\beta_k$ is a noise schedule, satisfying $\beta_k<\beta_{k-1}$.
As $K$ becomes large, $\mathbf{Y}_K \approx \mathcal{N}(0,\mathbf{I})$.
Following DDPM~\cite{ho2020denoising}, we sample $\mathbf{Y}_k$ from $\mathbf{Y}_0$ at any time step $k$ in a closed form:
\begin{equation}\label{eq:process_closed_form}
    q(\mathbf{Y}_k | \mathbf{Y}_{0}) = \mathcal{N}(\mathbf{Y}_k;\sqrt{\bar{\alpha}_k} \mathbf{Y}_{0}, (1-\bar{\alpha}_k) \mathbf{I}),
\end{equation}
where $\alpha_k = 1 - \beta_k$ and $\bar{\alpha}_k = \prod_{s=0}^k \alpha_s$.
Utilizing the parameterization trick~\cite{kingma2013auto}, we express $\mathbf{Y}_k$ as:
\begin{equation} \label{eq:Y_k}
    \mathbf{Y}_k = \sqrt{\bar{\alpha}_k} \mathbf{Y}_0 + \sqrt{1 - \bar{\alpha}_k} \epsilon,
\end{equation}
where $\epsilon \sim \mathcal{N}(0, \mathbf{I})$.  
The detailed derivations are provided in \textbf{Appendix}~\ref{detail_Y_k}.

\noindent\textbf{Reverse Process.}
The reverse process aims to estimate the posterior $q(\mathbf{Y}_{k-1}|\mathbf{Y}_k)$.
Different from PhysDiff~\cite{qian2025physdiff}, in our frequency-aware diffusion model, this distribution is approximated by a neural network $f_\theta$ conditioned on both the time-domain condition $\mathbf{X}$ and the physiological frequency condition $\mathbf{C}^{\mathbf{P}}$:
\begin{equation}\label{eq:reverse}
    p_{\theta}(\mathbf{Y}_{k-1}|\mathbf{Y}_k, \mathbf{X}, \mathbf{C}^{\mathbf{P}}) = \mathcal{N}(\mathbf{Y}_{k-1};\mu_\theta(\mathbf{Y}_k, \mathbf{X}, \mathbf{C}^{\mathbf{P}},k), \Sigma_{\theta}).
\end{equation}
Next, we show that incorporating the physiological frequency prior $\mathbf{C}^{\mathbf{P}}$ can effectively reduce the uncertainty in the reverse diffusion process, leading to more accurate rPPG signal reconstruction.

\noindent\textbf{Training.}
Traditional DDPM-based training involves learning to predict the added Gaussian noise at each diffusion step, which can be inefficient~\cite{ho2020denoising}.
Instead, our denoising network $f_\theta$ is designed to directly reconstruct the clean rPPG signal $\mathbf{Y}_0$ from the noisy input $\mathbf{Y}_0$, conditioned on $\mathbf{X}$, $\mathbf{C}^{\mathbf{P}}$, and the timestep $k$:
\begin{equation}
\hat{\mathbf{Y}}_0=f_\theta(\mathbf{Y}_k,\mathbf{X},\mathbf{C}^{\mathbf{P}},k).
\end{equation}
In practice, for Equation~\ref{eq:reverse}, the mean $\mu_\theta$ and covariance $\sigma_{k}^2$ in reverse process are parameterized as $\mu_\theta(\mathbf{Y}_k, \mathbf{X}, \mathbf{C}^{\mathbf{P}},k)$ and $\Sigma_{\theta}$ (details are presented in \textbf{Appendix}~\ref{appendix:reverse_parameters}).
Furthermore, inspired by the Fourier-based loss term, which is beneficial for the accurate reconstruction of the signals\cite{fons2022hypertime}, we propose to guide the diffusion training by applying it to the frequency domain with the Fourier transform.
Formally, our training objective integrates both time and frequency-domain constraints:
\begin{equation}
\begin{split}
    \mathcal{L}_{\theta}(\hat{\mathbf{Y}}_0,\mathbf{Y}_0)
    = \underbrace{1 - Pearson(\hat{\mathbf{Y}}_0,\mathbf{Y}_0)}_{\text{time-domain loss}} 
    \quad + \underbrace{MSE(\mathcal{F}(\hat{\mathbf{Y}}_0), \mathcal{F}(\mathbf{Y}_0))}_{\text{frequency-domain loss}},
\end{split}
\end{equation}
where $Pearson$ represents Pearson correlation coefficient, $MSE$ denotes Mean Square Error, and $\mathcal{F}$ denotes the Discrete Fourier Transform.
For inference, we start from $\mathbf{Y}_K \sim \mathcal{N}(0, \mathbf{I})$, $K$, $\mathbf{X}$, and $\mathbf{C}^{\mathbf{P}}$.
Then, we follow DDIM~\cite{songdenoising,qian2025physdiff} and perform the reverse process to obtain the final rPPG signal.

\section{Experiments}
\subsection{Experimental Setup}
\textbf{Datasets and Metrics.} We evaluate our method using both intra-dataset and cross-dataset protocols on six widely used rPPG benchmark datasets, including UBFC-rPPG~\cite{bobbia2019unsupervised}, PURE~\cite{stricker2014non}, MMPD~\cite{tang2023mmpd}, BUAA~\cite{xi2020image}, VIPL-HR~\cite{niu2019rhythmnet}, and MR-NIRP-Car~\cite{nowara2020near}. 
The detailed description of datasets and evaluation metrics is provided in \textbf{Appendix}~\ref{details_datasets} and \textbf{Appendix}~\ref{metric}, respectively.

\noindent\textbf{Implementation Details.}
The proposed denoising network consists of $L=4$ layers with a feature dimension of $D=128$.
The diffusion step is set to $K=1000$, following PhysDiff~\cite{qian2025physdiff}.
During training, the model is optimized using Adam for 50 epochs with an initial learning rate of $1\times10^{-3}$, and all experiments are conducted on four NVIDIA RTX 4090 GPUs (24GB).

\begin{table*}[t]
\caption{\textbf{Intra-dataset HR estimation results.} Models are trained and tested on the UBFC-rPPG, PURE, VIPL-HR, and MMPD datasets. {\textbf{Bold}}: best results.}
\vspace{-2.5ex}
\label{tab:intra-dataset}
\centering
\renewcommand\arraystretch{1.0}
\tabcolsep 1pt
\resizebox{1.0\linewidth}{!}{
\begin{tabular}{cc|ccc|ccc|ccc|ccc|ccc|ccc}
\toprule\toprule
\rowcolor[HTML]{f8f9fa} & & \multicolumn{3}{c|}{\textbf{UBFC-rPPG}} & \multicolumn{3}{c|}{\textbf{PURE}} & \multicolumn{3}{c|}{\textbf{BUAA}} & \multicolumn{3}{c|}{\textbf{MMPD}} & \multicolumn{3}{c|}{\textbf{VIPL-HR}} & \multicolumn{3}{c}{\textbf{MR-NIRP-Car}}\\ 
\rowcolor[HTML]{f8f9fa} \multirow{-2}{*}{\textbf{Method}} & \multirow{-2}{*}{\textbf{Venue}} & MAE$\downarrow$ & RMSE$\downarrow$ & $r\uparrow$ & MAE$\downarrow$ & RMSE$\downarrow$ & $r\uparrow$ & MAE$\downarrow$ & RMSE$\downarrow$ & $r\uparrow$ & MAE$\downarrow$ & RMSE$\downarrow$ & $r\uparrow$ & MAE$\downarrow$ & RMSE$\downarrow$ & $r\uparrow$ & MAE$\downarrow$ & RMSE$\downarrow$ & $r\uparrow$\\ 
\midrule
\multicolumn{20}{l}{\emph{\textcolor{gray}{Traditional rPPG Methods}}}\\
GREEN~\cite{verkruysse2008remote} & Opt Express'08 & 7.50 & 14.41  & 0.62 & 7.23 & 17.05 & 0.69 & 6.89 & 10.39 & 0.60 & 21.68  & 27.69 & -0.01 & - & - & - & 11.88 & 14.46 & 0.17\\
ICA~\cite{poh2010non} & Opt Express'10 & 5.17 & 11.76 & 0.65 & 3.76 & 12.60 & 0.85 & - & - & - & 18.60 & 24.30 & 0.01 & - & - & - & 10.87 & 13.62 & 0.02\\
CHROM~\cite{de2013robust} & TBE'13 & 2.37 & 4.91 & 0.89 & 2.07 & 2.50 & 0.99 & - & - & - & 13.66 & 18.76 & 0.08 & 11.40 & 16.90 & 0.28 & 7.16 & 9.09 & 0.35\\
\midrule
\multicolumn{20}{l}{\emph{\textcolor{gray}{Deep Learning-based rPPG Methods}}}\\
DeepPhys~\cite{chen2018deepphys} & ECCV'18 & 6.27 & 10.82 & 0.65 & 0.83 & 1.54 & {{0.99}} & - & - & - & 22.27 & 28.92 & -0.03 & 11.0 & 13.8 & 0.72 & 11.98 & 14.18 & 0.05\\
PhysNet~\cite{yu2019remote1} & BMVC'19 & 2.95 & 3.67 & 0.97 & 2.10 & 2.60 & {{0.99}} & 10.89 & 11.70 & -0.04 & 4.80 & 11.80 & 0.60 & 10.80 & 14.80 & 0.20 & 11.18 & 13.53 & 0.17\\
CVD~\cite{niu2020video} & ECCV'20 & 2.19 & 3.12 & {{0.99}} & 1.29 & 2.01 & 0.98 & - & - & - & - & - & - & 5.02 & 7.97 & 0.79 & - & - & -\\
TS-CAN~\cite{liu2020multi} & NeurIPS'20 & 1.70 & 2.72 & {{0.99}} & 2.48 & 9.01 & 0.92 & - & - & - & 9.71 & 17.22 & 0.44 & - & - & - & 13.62 & 15.46 & 0.19\\
Gideon et al.\cite{gideon2021way} & ICCV'21 & 1.85 & 4.28 & 0.93 & 2.30 & 2.90 & {{0.99}} & - & - & - & - & - & - & 9.01 & 14.02 & 0.58 & - & - & -\\
Dual-GAN~\cite{lu2021dual} & CVPR'21 & 0.44 & 0.67 & {{0.99}} & 0.82 & 1.31 & {{0.99}} & - & - & - & - & - & - & 4.93 & 7.68 & 0.81 & - & - & - \\
PhysFormer~\cite{yu2022physformer} & CVPR'22 & 0.50 & 0.71 & {{0.99}} & 1.10 & 1.75 & {{0.99}} & 8.46 & 10.17 & -0.06 & 11.99 & 18.41 & 0.18 & 4.97 & 7.79 & 0.78 & 10.17 & 12.57 & -0.20\\ 
Contrast-Phys~\cite{sun2022contrast} & ECCV'22 & 0.64 & 1.00 & {{0.99}} & 1.00 & 1.40 & {{0.99}} & - & - & - & - & - & - & 32.1 & 36.1 & 0.04 & - & - & -\\
EfficientPhys~\cite{liu2023efficientphys} & WACV'23 & 1.14 & 1.81 & {{0.99}} & 4.75 & 9.39 & 0.99 & 16.09 & 16.80 & 0.14 & 13.47 & 21.32 & 0.21 & - & - & - & 18.36 & 20.01 & -0.04\\
Li et al.\cite{li2023contactless} & ICCV'23 & 0.48 & {0.64} & {\textbf{1.00}} & 0.64 & 1.16 & {{0.99}} & - & - & - & - & - & - & 4.97 & 7.79 & 0.78 & - & - & -\\
Yue et al.\cite{yue2023facial} & TPAMI'23 & 0.58 & 0.94 & {{0.99}} & 1.23 & 2.01 & {{0.99}} & - & - & - & - & - & -  & - & - & -\\
Contrast-Phys+\cite{sun2024contrast} & TPAMI'24 & {{0.21}} & 0.80 & {{0.99}} & {0.48} & {0.98} & {{0.99}} & - & - & - & - & - & -  & - & - & - & - & - & -\\
CodePhys~\cite{chu2025codephys} & JBHI'25 & 0.21 & 0.26 & 0.99 & 0.39 & 0.83 & 0.99 & - & - & - & - & - & - & 4.27 & 7.11 & 0.81 & - & - & -\\
RhythmMamba~\cite{zou2025rhythmmamba} & AAAI'25 & 0.50 & 0.75 & {{0.99}} & {{0.23}} & {{0.34}} & {{0.99}} & 11.99 & 14.22 & 0.16 & {\textbf{3.16}} & {{7.27}} & {{0.84}} & 4.30 & 7.49 & 0.81 & 8.07 & 10.03 & 0.24\\
PhysDiff~\cite{qian2025physdiff} & AAAI'25 & {0.33} & {{0.57}} & {\textbf{1.00}} & {0.29} & {0.54} & {\textbf{1.00}} & - & - & - & 7.17 & 9.63 & 0.71 & {{3.92}} & {{6.65}} & {{0.85}} & 7.58 & 8.58 & -0.14\\
PhysLLM~\cite{xie2025physllm} & ICLR'26 & {0.21} & 0.57 & 0.99 & {0.17} & 0.35 & 0.99 & 6.48 & 8.48 & 0.63 & 4.36 & 10.76 & 0.65 & 4.24 & 6.81 & - & - & - & - \\
PHASE-Net~\cite{zhao2025phase} & CVPR'26 & \textbf{0.15} & 0.53 & 0.99 & \textbf{0.14} & 0.35 & 0.99 & 5.89 & 7.89 & 0.48 & 4.78 & 8.22 & 0.71 & - & - & - & - & - & - \\
\midrule
\rowcolor[HTML]{EBFEE8} \textcolor{SkyBlue}{\textbf{FreqPhys}} \textbf{(Ours)} & - & {0.17} & {\textbf{0.41}} & {\textbf{1.00}} & {{0.17}} & {\textbf{0.25}} & {\textbf{1.00}} & \textbf{1.13} & \textbf{1.31} & \textbf{0.99} & {{4.20}} & {\textbf{6.78}} & {\textbf{0.86}} & {\textbf{3.79}} & {\textbf{6.34}} & {\textbf{0.86}} & \textbf{5.75} & \textbf{6.08} & \textbf{0.54}\\
\bottomrule\bottomrule
\end{tabular}}
\end{table*}

\begin{table*}[t]
\caption{\textbf{HRV and RF estimation results.} Models are trained and tested on the UBFC-rPPG dataset. LF, HF, and RF represent low frequency, high frequency, and respiration frequency, respectively. ``n.u.'' denotes normalized units.}
\vspace{-2.5ex}
\renewcommand\arraystretch{1.0}
\tabcolsep 7pt
\centering
\label{tab:hrv testing}
\resizebox{1.0\linewidth}{!}{
\begin{tabular}{cc|ccc|ccc|ccc|ccc}
\toprule\toprule
\rowcolor[HTML]{f8f9fa} & &  \multicolumn{3}{c|}{\textbf{LF (n.u.)}} &  \multicolumn{3}{c|}{\textbf{HF (n.u)}} &  \multicolumn{3}{c|}{\textbf{LF/HF}} &  \multicolumn{3}{c}{\textbf{RF (Hz)}}\\
\rowcolor[HTML]{f8f9fa} \multirow{-2}{*}{\textbf{Method}} & \multirow{-2}{*}{\textbf{Venue}} &  \multicolumn{1}{c}{SD$\downarrow$} & RMSE$\downarrow$ & \multicolumn{1}{c|}{$r\uparrow$}  &  \multicolumn{1}{c}{SD$\downarrow$} & RMSE$\downarrow$ & \multicolumn{1}{c|}{$r\uparrow$}  & \multicolumn{1}{c}{SD$\downarrow$} & RMSE$\downarrow$ & \multicolumn{1}{c|}{$r\uparrow$}  &  \multicolumn{1}{c}{SD$\downarrow$} & RMSE$\downarrow$ & \multicolumn{1}{c}{$r\uparrow$}\\ \midrule
\multicolumn{14}{l}{\emph{\textcolor{gray}{Traditional rPPG Methods}}}\\
GREEN~\cite{verkruysse2008remote} & Opt Express'08 & 0.186 & 0.186  & 0.280  & 0.186  & 0.186 & 0.280 & 0.361  & 0.365 & 0.492 & 0.087  & 0.086 & 0.111\\
ICA~\cite{poh2010non} & Opt Express'10 &0.243	&0.240	&0.159	&0.243	&0.240	&0.159	&0.655	&0.645	&0.226	&0.086&	0.089 &0.102\\
POS~\cite{wang2016algorithmic} & TBE'16 &0.171	&0.169	&0.479	&0.171	&0.169	&0.479	&0.405	&0.399	&0.518	&0.109	&0.107	& 0.087\\ \midrule
\multicolumn{14}{l}{\emph{\textcolor{gray}{Deep Learning-based rPPG Methods}}}\\
CVD~\cite{niu2020video} & ECCV'20   &0.053 & 0.056 & 0.740 & 0.053 & 0.065 & 0.740 & 0.169 & 0.168 & 0.812 & 0.017 & 0.018 & 0.252\\
Dual-GAN~\cite{lu2021dual} & CVPR'21 &{0.034} & {0.035} & {0.891} & {0.034} & {0.034} & {0.891} & {0.131} & 0.136 & {0.881}& 0.010 & 0.010 & {0.395}\\
Gideon \etal ~\cite{gideon2021way} & ICCV'21 & 0.091 & 0.139 & 0.694 & 0.091 & 0.139 & 0.694 & 0.525 & 0.691 & 0.684 & 0.061 & 0.098 & 0.103\\
Contras-Phys~\cite{sun2022contrast} & ECCV'22  & 0.050 & 0.098 & 0.798 & 0.050 & 0.098 & 0.798 & 0.205 & 0.395 & 0.782 & 0.055 & 0.083 & 0.347\\
Contrast-Phys+ ~\cite{sun2024contrast}  & TPAMI'24 & {{0.025}} & 0.025 & 0.947 & 0.025 & 0.025 & 0.947 &{\textbf{0.064}} & {\textbf{0.066}} & 0.963 &0.029 & 0.029 & 0.803\\
PhysDiff~\cite{qian2025physdiff} & AAAI'25 & 0.029 & {{0.022}} & {{0.978}} & {{0.016}} & {{0.022}} & {{0.978}} & 0.079 & {{\textbf{0.066}}} & {{{0.979}}} & {{\textbf{0.006}}} & {{{0.007}}} & {{{0.811}}}\\ \midrule
\rowcolor[HTML]{EBFEE8} \textcolor{SkyBlue}{\textbf{FreqPhys}} \textbf{(Ours)} & - & {\textbf{0.016}} & {{\textbf{0.014}}} & {{\textbf{0.988}}} & {{\textbf{0.013}}} & {{\textbf{0.014}}} & {{\textbf{0.988}}} & {{0.079}} & {{\textbf{0.066}}} & {{\textbf{0.989}}} & {{\textbf{0.006}}} & {{\textbf{0.005}}} & {{\textbf{0.845}}}\\
\bottomrule\bottomrule
\end{tabular}}
\end{table*}

\subsection{Intra-dataset Evaluation}
Tab.~\ref{tab:intra-dataset} reports the quantitative comparison with state-of-the-art methods. 
Notably, the evaluated datasets exhibit progressively increasing levels of difficulty. 
UBFC-rPPG and PURE are relatively controlled environments, whereas BUAA introduces extreme illumination variations. 
MMPD combines illumination changes with head motion and skin-tone diversity. 
The VIPL-HR dataset further increases the diversity of camera equipment and is currently the most complex dataset collected in laboratory settings.
Among them, MR-NIRP-Car represents the most challenging real-world scenario, as it is collected in a natural driving environment with substantial motion, illumination fluctuation, and environmental interference.
Despite these challenges, our method consistently achieves superior performance, with particularly large improvements on the more challenging complex datasets. 
These results demonstrate that incorporating physiological frequency priors enables our model to effectively suppress noise induced by illumination changes, motion artifacts, and cross-subject appearance variations.

In addition to HR estimation, we follow~\cite{niu2020video,lu2021dual,sun2022contrast,sun2024contrast,qian2025physdiff} and further evaluate our method on two other critical physiological indicators, i.e., HRV and RF, both of which require high-quality rPPG signals for reliable peak detection and temporal analysis.
As shown in Tab.~\ref{tab:hrv testing}, our method significantly outperforms existing methods across most metrics.

\begin{table*}[t]
\caption{\textbf{Cross-dataset HR estimation results.} Models are trained on PURE/UBFC-rPPG, and tested on UBFC-rPPG/PURE/MMPD.}
\vspace{-2.5ex}
\label{tab:cross-dataset}
\centering
\renewcommand\arraystretch{1.0}
\tabcolsep 7pt
\resizebox{1.0\linewidth}{!}{
\begin{tabular}{cc|ccc|ccc|ccc|ccc}
\toprule\toprule
\rowcolor[HTML]{f8f9fa} & & \multicolumn{3}{c|}{\textbf{PURE $\rightarrow$ UBFC-rPPG}} & \multicolumn{3}{c|}{\textbf{UBFC-rPPG $\rightarrow$ PURE}} & \multicolumn{3}{c|}{\textbf{PURE $\rightarrow$ MMPD}} & \multicolumn{3}{c}{\textbf{UBFC-rPPG $\rightarrow$ MMPD}}\\ 
\rowcolor[HTML]{f8f9fa} \multirow{-2}{*}{\textbf{Method}} & \multirow{-2}{*}{\textbf{Venue}} & MAE$\downarrow$ & RMSE$\downarrow$ & $r\uparrow$ & MAE$\downarrow$ & RMSE$\downarrow$ & $r\uparrow$ & MAE$\downarrow$ & RMSE$\downarrow$ & $r\uparrow$ & MAE$\downarrow$ & RMSE$\downarrow$ & $r\uparrow$ \\ \midrule
\multicolumn{14}{l}{\emph{\textcolor{gray}{Traditional rPPG Methods}}}\\
GREEN~\cite{verkruysse2008remote} & Opt Express'08 & 19.73 & 21.00  & 0.37  & 10.09  & 23.85 & 0.34 & 21.68  & 27.69 & -0.01 & 21.68  & 27.69 & -0.01\\
ICA~\cite{poh2010non} & Opt Express'10 & 16.00	& 25.65	& 0.44	&4.77 & 16.07 & 0.72 & 18.06 & 24.30 & 0.01	& 18.06 & 24.30 & 0.01\\
CHROM~\cite{de2013robust} & TBE'13 & 4.06 & 8.83 & 0.89	& 5.77	& 14.93	& 0.81	& 13.66	& 18.76 & 0.08 &13.66	& 18.76 & 0.08\\ \midrule
\multicolumn{14}{l}{\emph{\textcolor{gray}{Deep Learning-based rPPG Methods}}}\\
DeepPhys~\cite{chen2018deepphys} & ECCV'18 & 1.21 & 2.90 & {{0.99}} & 5.54 & 18.51 & 0.66 & 16.92 & 24.61 & 0.05 & 17.50 & 25.00 & 0.05 \\
PhysNet~\cite{yu2019remote1} & BMVC'19 & 1.63 & 3.79 & 0.98 & 9.36 & 20.63 & 0.62 & 13.22 & 19.61 & 0.23 & {{10.24}} & 16.54 & 0.29 \\
TS-CAN~\cite{liu2020multi} & NeurIPS'20 & 1.30 & 2.87 & {{0.99}} & 3.69 & 13.80 & 0.82 & 13.94 & 21.61 & 0.20 & 14.01 & 21.04 & 0.24 \\
PhysFormer~\cite{yu2022physformer} & CVPR'22 & 1.44 & 3.77 & 0.98 & 12.92 & 24.36 & 0.47 & 14.57 & 20.71 & 0.15 & 12.10 & 17.79 & 0.17\\ 
EfficientPhys~\cite{liu2023efficientphys} & WACV'23 & 2.13 & 3.00 & {{0.99}} & 5.47 & 17.04 & 0.71 & 14.03 & 21.62 & 0.17 & 13.78 & 22.25 & 0.09\\
RhythmMamba~\cite{zou2025rhythmmamba} & AAAI'25 & 0.95 & 1.83 & {{0.99}} & {{1.98}} & {{6.51}} & {{0.96}} & {{10.44}} & 16.70 & {{0.36}} & 10.63 & 17.14 & {{0.34}}\\
\rowcolor[HTML]{EBFEE8} \textcolor{SkyBlue}{\textbf{FreqPhys}} \textbf{(Ours)} & - & {\textbf{0.43}} & {\textbf{0.79}} & {\textbf{1.00}} & {\textbf{0.95}} & {\textbf{3.15}} & {\textbf{0.99}} & {\textbf{10.11}} & {\textbf{14.34}} & {\textbf{0.48}} & {\textbf{8.91}} & {\textbf{12.86}} & {\textbf{0.57}}\\
\bottomrule\bottomrule
\end{tabular}}
\end{table*}

\subsection{Cross-dataset Evaluation}
In addition to intra-dataset evaluation, we further conduct cross-dataset experiments to simulate deployment in unseen environments and rigorously assess domain invariance. 
As shown in Tab.~\ref{tab:cross-dataset}, we evaluate four cross-dataset transfer settings. 
These datasets exhibit significant domain gaps due to differences in illumination conditions, motion patterns, camera devices, and subject appearance, making cross-dataset generalization particularly challenging. 
Consequently, most existing methods experience notable performance degradation when transferring from a relatively simple source domain (e.g., PURE or UBFC-rPPG) to more complex target domains such as MMPD. 
In contrast, benefiting from physiological spectrum modeling, our method effectively captures domain-invariant physiological priors and consistently achieves the best performance across all transfer settings, demonstrating stronger robustness and generalization capability under substantial cross-dataset domain shifts.

More extended cross-dataset testing results are provided in \textbf{Appendix~\ref{appendix_experiments}}.

\begin{table}[t]
\begin{minipage}[l]{0.43\linewidth}
\centering
\renewcommand\arraystretch{1.0}
\caption{{Impact of physiological frequency denoiser components.}}
\vspace{-2.5ex}
\resizebox{1.0\linewidth}{!}{
\begin{tabular}{ccc|ccc|ccc}
\toprule
\rowcolor[HTML]{f8f9fa} \multirow{2}{*}{\textbf{PBF}} & \multirow{2}{*}{\textbf{PSM}} & \multirow{2}{*}{\textbf{ASS}} & \multicolumn{3}{c|}{\textbf{VIPL-HR}} & \multicolumn{3}{c}{\textbf{MR-NIRP-Car}}\\ 
& & & MAE$\downarrow$ & RMSE$\downarrow$ & $r\uparrow$ & MAE$\downarrow$ & RMSE$\downarrow$ & $r\uparrow$\\
\midrule
\checkmark & - & - & 4.03 & 6.77 & 0.82 & 6.59 & 8.46 & 0.29\\
- & \checkmark & - & 3.98 & 6.55 & 0.84 & 6.24 & 7.62 & 0.32\\
- & - & \checkmark & 4.10 & 6.95 & 0.82& 6.81 & 8.03 & 0.26\\
\checkmark & \checkmark & - & 3.91 & 6.42 & 0.85 & 6.23 & 7.11 & 0.38\\
\checkmark & - & \checkmark & 4.11 & 6.71 & 0.84 & 6.67 & 7.93 & 0.31\\
\rowcolor[HTML]{EBFEE8} \checkmark & \checkmark & \checkmark & \textbf{3.79} & \textbf{6.34} & \textbf{0.86} & \textbf{5.75} & \textbf{6.08} & \textbf{0.54}\\
\bottomrule
\end{tabular}}
\label{tab:components}
\end{minipage}
\hfill
\begin{minipage}[l]{0.55\linewidth}
\centering
\renewcommand\arraystretch{1.0}
\caption{{Generality of physiological frequency denoiser (PFD) module}.}
\vspace{-2.5ex}
\resizebox{1.0\linewidth}{!}{
\begin{tabular}{c|ccc|ccc}
\toprule
\rowcolor[HTML]{f8f9fa} \multirow{2}{*}{\textbf{Method}} & \multicolumn{3}{c|}{\textbf{VIPL-HR}} & \multicolumn{3}{c}{\textbf{MR-NIRP-Car}}\\ 
& MAE$\downarrow$ & RMSE$\downarrow$ & $r\uparrow$ & MAE$\downarrow$ & RMSE$\downarrow$ & $r\uparrow$\\
\midrule
RhythmMamba~\cite{zou2025rhythmmamba} & 4.30 & 7.49 & 0.81 & 8.07 & 10.03 & 0.24\\
\rowcolor[HTML]{EBFEE8} RhythmMamba~\cite{zou2025rhythmmamba}+\textbf{PFD (Ours)} & \textbf{4.15} & \textbf{7.13} & \textbf{0.82} & \textbf{7.88} & \textbf{9.21} & \textbf{0.29}\\
PhysDiff~\cite{qian2025physdiff} & 3.92 & 6.65 & 0.85 & 7.40 & 8.39 & 0.31\\
\rowcolor[HTML]{EBFEE8} PhysDiff~\cite{qian2025physdiff}+\textbf{PFD (Ours)} & \textbf{3.84} & \textbf{6.44} & \textbf{0.86} & \textbf{7.04} & \textbf{8.13} & \textbf{0.34}\\
\textcolor{SkyBlue}{\textbf{FreqPhys}} w/o \textbf{PFD (Ours)} & 4.22 & 7.17 & 0.83& 5.93 & 6.55 & 0.46\\
\rowcolor[HTML]{EBFEE8} \textcolor{SkyBlue}{\textbf{FreqPhys}} & \textbf{3.79} & \textbf{6.34} & \textbf{0.86} & \textbf{5.75} & \textbf{6.08} & \textbf{0.54}\\
\bottomrule
\end{tabular}}
\label{tab:generality}
\end{minipage}
\end{table}

\subsection{Ablation Studies}\label{sec:ablation}
\noindent\textbf{How effective are the components of the physiological frequency denoiser?}
To investigate the contribution of each component in the proposed physiological frequency denoiser, we conduct detailed ablation studies on the three modules: \textit{Physiological Bandpass Filter} (PBF), \textit{Physiological Spectrum Modulation} (PSM), and \textit{Adaptive Spectrum Selection} (ASS). 
As shown in Tab.~\ref{tab:components}, each module improves the performance, and their combination yields the best results, indicating that these components work synergistically to suppress spectral noise and enhance physiologically meaningful frequency representations. 
Furthermore, as reported in Tab.~\ref{tab:generality}, when the physiological frequency denoiser module is removed and only the time-domain MSTmap is used as the conditional input, the performance drops noticeably, further demonstrating the importance of incorporating physiological frequency priors for robust rPPG estimation.

\noindent\textbf{How general is the physiological frequency denoiser?}
In addition to analyzing the contribution of individual components, we further investigate the generality of the proposed \textit{Physiological Frequency Denoiser} (PFD). 
It is designed as a modular plug-in that can be seamlessly integrated into existing rPPG architectures. 
As shown in Tab.~\ref{tab:generality}, we evaluate the complete PFD module (PBF+PSM+ASS) when incorporated into two state-of-the-art models, RhythmMamba~\cite{zou2025rhythmmamba} and PhysDiff~\cite{qian2025physdiff}. 
Specifically, PFD is inserted before each {Multi-temporal Constraint Mamba} block in RhythmMamba and before each {Spatial–Temporal Hybrid Attention} block in PhysDiff, while keeping all other settings unchanged. 
Across all evaluation metrics on the VIPL-HR dataset, both backbones achieve consistent and notable performance gains after integrating PFD. 
These results demonstrate the strong generality of PFD, indicating that it can be seamlessly integrated into diverse backbone architectures while providing cleaner and more physiologically meaningful spectral clues for rPPG estimation.

\begin{wraptable}[10]{r}{6.1cm}
\centering
\vspace{-6ex}
\caption{\textbf{Computational cost}. Best results are marked in bold and the second best in underline.}
\scriptsize
\resizebox{1.0\linewidth}{!}{
\begin{tabular}{l|ccccc}
\toprule
\multirow{2}{*}{\textbf{Method}} & Parameters & Flops & Throughput & Inference & Memory\\ 
& (M)$\downarrow$ & (G)$\downarrow$ & (Kps)$\uparrow$ & time (ms)$\downarrow$ & (M)$\downarrow$\\
\midrule
DeepPhys~\cite{chen2018deepphys} & 1.98 & 111.67 & 28.89 & 34.61 & 10638\\
PhysNet~\cite{yu2019remote1} & \textbf{0.77} & 65.74 & \underline{68.73} & \underline{14.55} & 3750\\
TS-CAN~\cite{liu2020multi} & 1.98 & 111.67 & 26.23 & 38.13 & 11834\\
PhysFormer~\cite{yu2022physformer} & 7.38 & 47.44 & 50.79 & 19.69 & 6480\\
EfficientPhys~\cite{liu2023efficientphys} & 1.91 & 56.06 & 41.36 & 24.18 & 7814\\
RhythmMamba~\cite{zou2025rhythmmamba} & 2.00 & \underline{12.41} & 27.16 & 36.82 & 2450\\
PhysDiff~\cite{qian2025physdiff} & 2.64 & {22.46} & 60.23 & 16.60 & \underline{1246}\\
\rowcolor[HTML]{EBFEE8} \textcolor{SkyBlue}{\textbf{FreqPhys}} \textbf{(ours)} & \underline{0.86} & \textbf{7.75} & \textbf{85.44} & \textbf{11.70} & \textbf{874}\\
\bottomrule
\end{tabular}}
\label{tab:computational_cost}
\end{wraptable}
\subsection{Computational Cost}
To evaluate the deployment efficiency of our method, we conduct a unified 10-second inference benchmark on a single NVIDIA RTX 4090 GPU (24GB). 
The detailed computational statistics are reported in Tab.~\ref{tab:computational_cost}. 
Notably, compared with the lightweight RhythmMamba, our method consistently achieves lower computational overhead across all key metrics. 
Specifically, our \textcolor{SkyBlue}{\textbf{FreqPhys}} reduces the parameters from 2.00M to 0.86M (57.0\% reduction) and decreases FLOPs from 12.41G to 7.75G (37.6\% reduction). 
In addition, our method significantly improves runtime efficiency, increasing throughput from 27.16 Kps to 85.44 Kps (3.1$\times$ faster) and reducing inference latency from 36.82 ms to 11.70 ms. 
Meanwhile, the peak GPU memory usage is reduced from 2450 MB to 874 MB (64.3\% reduction). 
These results demonstrate that our method achieves substantially lower computational cost, making it highly suitable for real-time and resource-constrained deployment scenarios.

\begin{figure}[t]
\centering
\includegraphics[width=0.9\linewidth]{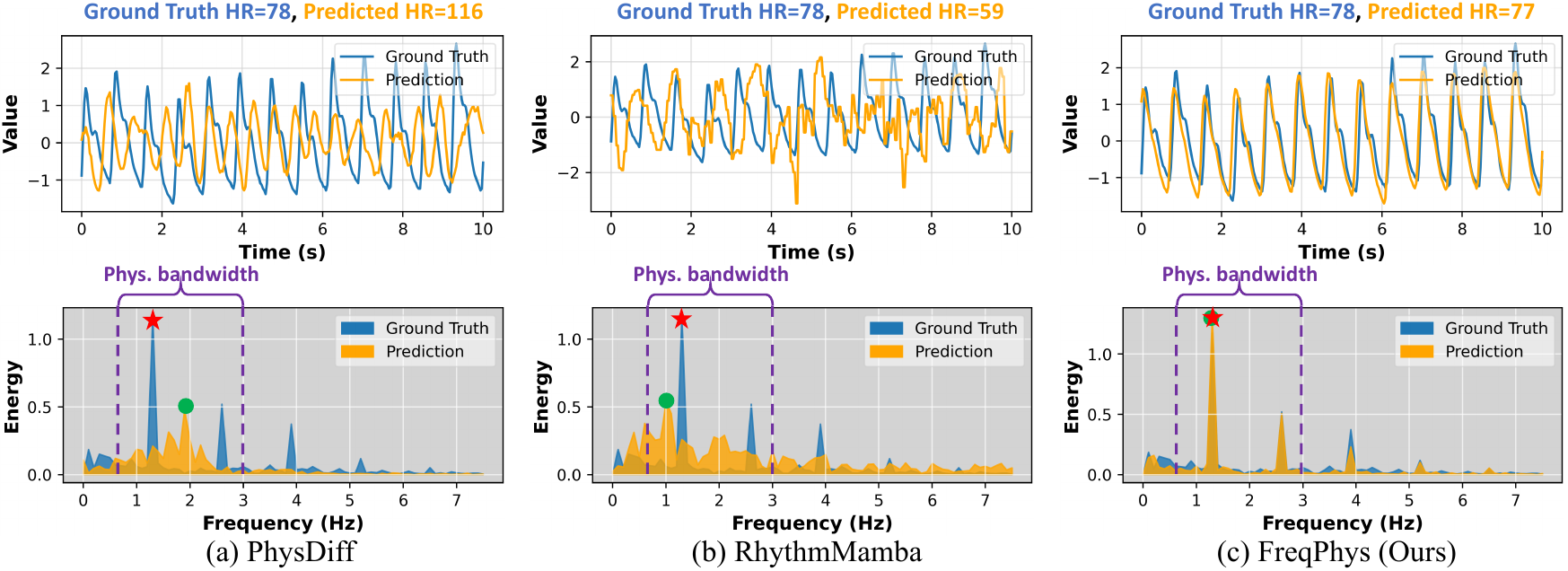}
\vspace{-1ex}
\caption{Time and frequency domain visualizations of rPPG signal predictions on the VIPL dataset under head motion scenario.
In the frequency-domain plots, the purple dashed box indicates the physiological signal bandwidth ranging from 0.66 to 3.0 Hz, corresponding to typical human cardiac frequencies. $\textcolor{red}{\star}$ and $\textcolor{mygreen}{\bullet}$ denote the spectral peaks of the ground-truth and predicted heart rates, respectively.}
\label{fig:rppg_visualization}
\end{figure}

\subsection{Qualitative Analysis}
\textbf{Visualization of rPPG Prediction.}\label{sec:visualization}
We visualize the rPPG predictions to highlight the improvements of our method in rPPG estimation quality.
We present a prediction showcase on the VIPL dataset under the head motion scenario, as shown in Fig.~\ref{fig:rppg_visualization}.
PhysDiff~\cite{qian2025physdiff}, RhythmMamba~\cite{zou2025rhythmmamba} are selected as the representative SOTA methods.
We observe that the rPPG signal predicted by PhysDiff and RhythmMamba exhibits numerous burrs, which reflect the limitations of modeling in the time domain, namely, the difficulty in capturing dominant physiological frequency components and being vulnerable to noise.
Our \textcolor{SkyBlue}{\textbf{FreqPhys}} addresses this limitation effectively, which not only keeps pace with the label sequence but also accurately exhibits a smoother appearance with fewer irregularities.
In addition, we also observe that the frequency spectra of PhysDiff and RhythmMamba exhibit relatively dispersed energy distributions. 
In contrast, our method effectively concentrates spectral energy on several key frequency components corresponding to physiological signals, resulting in an obvious sparsity in the frequency domain. 

\begin{wrapfigure}[10]{r}{6.2cm}
\centering
\vspace{-5ex}
\includegraphics[width=1.0\linewidth]
{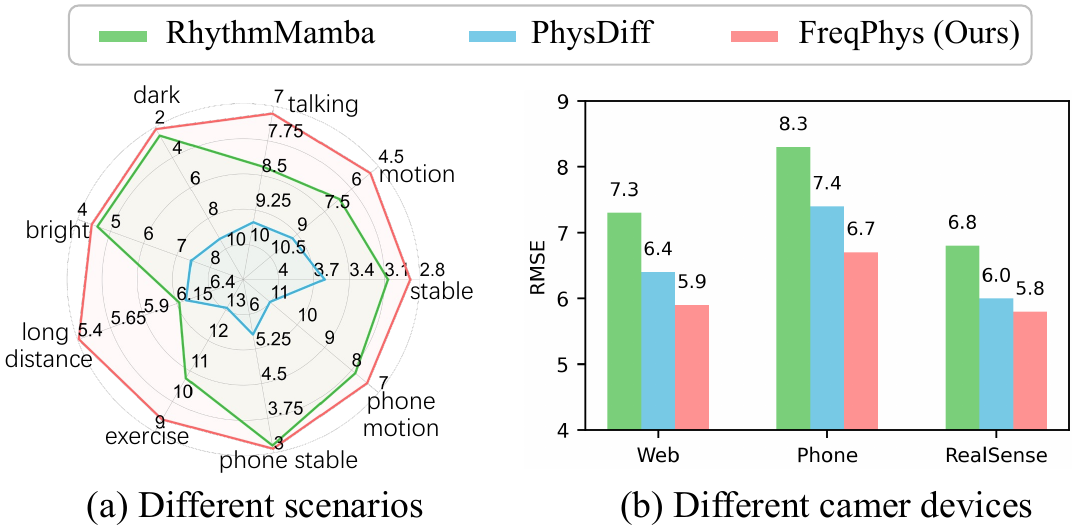}
\vspace{-4ex}
\caption{HR estimation results on VIPL-HR under different scenarios and cameras.}
\label{fig:different_scenarios_sources_ranges}
\end{wrapfigure}
\noindent\textbf{Qualitative Results for Robustness.}
To evaluate the robustness of our model across diverse scenarios and camera device conditions, we provide detailed results from the VIPL-HR dataset, encompassing 9 distinct scenarios and 3 types of camera devices. 
As illustrated in Fig.~\ref{fig:different_scenarios_sources_ranges}, our proposed method consistently outperforms other methods across these varied conditions, which indicates its robustness in remote physiological signal measurement.

\section{Conclusion}
In this paper, we introduced \textcolor{SkyBlue}{\textbf{FreqPhys}}, a diffusion-based framework that leverages implicit physiological frequency priors for robust remote photoplethysmography. 
Unlike existing approaches that rely primarily on time-domain modeling or incorporate spectral clues only at the loss or representation level, our method integrates physiological spectral priors directly into the iterative diffusion denoising process. 
This design enables more effective suppression of motion-induced artifacts while preserving meaningful cardiac rhythms. 
Extensive experiments on four public datasets demonstrate that our method achieves state-of-the-art performance in HR, HRV, and RF estimation, while also exhibiting strong cross-dataset generalization. 

\bibliographystyle{splncs04}
\bibliography{main}

\clearpage
\appendix
\setcounter{section}{0} 
\renewcommand{\thesection}{\Alph{section}}

\section{Detailed Related Work} \label{appendix_related_work}
\subsection{Remote Physiological Measurement}
Remote photoplethysmography (rPPG) relies on the physiological phenomenon that periodic cardiac contraction and relaxation induce subtle fluctuations in subcutaneous blood volume. 
These variations modulate the absorption and scattering properties of the skin, leading to slight color changes that can be captured by cameras, although they are imperceptible to the human eye.
With the success of deep learning (DL), data-driven approaches have become the dominant paradigm for rPPG estimation. 
HR-CNN~\cite{vspetlik2018visual} represents one of the earliest attempts to apply DL to rPPG by introducing a two-stage convolutional neural network for heart rate estimation from facial videos. 
PhysNet~\cite{yu2019remote1} subsequently proposed an end-to-end spatio-temporal network that directly reconstructs rPPG signals from video sequences, enabling accurate heart rate and heart rate variability estimation through dedicated temporal modeling and signal-level supervision. 
DeepPhys~\cite{chen2018deepphys} further introduced a motion representation based on normalized frame differences and employed attention mechanisms to guide the extraction of physiological signals from raw facial videos. 
TS-CAN~\cite{liu2020multi} improved efficiency and accuracy by integrating temporal shift modules with attention mechanisms and multi-task learning, allowing the network to simultaneously estimate blood volume pulse and respiration signals with low computational cost. 
CVD~\cite{niu2020video} proposed a multi-scale spatial–temporal map representation and introduced a cross-verified feature disentangling strategy to separate physiological information from noise. 
Dual-GAN~\cite{lu2021dual} further explored adversarial learning by employing two GANs to jointly model BVP prediction and noise distribution, improving robustness across different facial regions.

To better capture long-range temporal dependencies, several studies have adopted Transformer-based architectures, including PhysFormer~\cite{yu2022physformer}, EfficientPhys~\cite{liu2023efficientphys}, Dual-TL~\cite{qian2024dual}, and RhythmFormer~\cite{zou2025rhythmformer}, which leverage self-attention mechanisms to model global temporal relationships in rPPG signals. 
More recently, state-space modeling has been explored for efficient long-sequence processing. 
Some Mamba-based frameworks, such as RhythmMamba~\cite{zou2025rhythmmamba} and PhysMamba~\cite{luo2024physmamba}, incorporate state-space models to achieve linear computational complexity while maintaining strong performance with reduced memory consumption and faster inference speed. 
Beyond architectural improvements, several works aim to enhance robustness and generalization in rPPG estimation. 
For example, NEST~\cite{lu2023neuron} introduces a neuron structure modeling strategy to improve domain generalization by maximizing feature-space coverage during training, thereby reducing under-optimized feature activation under unseen environmental conditions. 
PHASE-Net~\cite{zhao2025phase} proposes a physics-grounded framework derived from hemodynamic principles, modeling the pulse signal as a second-order dynamical system and implementing an efficient gated temporal convolutional network for accurate physiological signal recovery. 
Recently, cross-modal learning has also been explored for rPPG analysis. 
PhysLLM~\cite{xie2025physllm} integrates large language models with rPPG-specific modules through a collaborative optimization framework, enabling cross-modal reasoning and improved robustness by incorporating physiological knowledge and environmental context into the signal estimation process. 
In addition, the scarcity of labeled physiological data has motivated the development of self-supervised learning strategies. 
Contrastive learning approaches~\cite{gideon2021way,sun2024contrast} and masked autoencoding methods~\cite{liu2024rppg,speth2023non} enable robust representation learning from large amounts of unlabeled facial videos, demonstrating promising potential for real-world rPPG applications.

\subsection{Diffusion Model for rPPG Estimation}
Diffusion models have recently emerged as powerful generative frameworks that progressively corrupt training data with noise and learn to reverse this process to recover clean samples. 
Originally proposed for image generation~\cite{ho2020denoising}, denoising diffusion probabilistic models (DDPMs) have since been extended to a wide range of applications, including cross-modal generation~\cite{avrahami2022blended, fan2022frido}, video editing~\cite{ceylan2023pix2video}, and object detection~\cite{chen2022diffusiondet}. 
Their strong capability to model complex data distributions and perform iterative denoising makes them particularly suitable for recovering structured signals from heavily corrupted observations.

In the field of remote physiological measurement, the presence of motion artifacts, illumination variations, and extremely weak physiological signals poses significant challenges for reliable signal extraction. 
These characteristics naturally motivate the use of diffusion-based models, which can progressively remove noise and recover latent physiological signals. 
A recent pioneering work, PhysDiff~\cite{qian2025physdiff}, introduces diffusion modeling to the rPPG task by proposing a dynamic-aware signal representation. 
Specifically, PhysDiff decomposes rPPG signals into two components: a trend term that captures temporal directionality associated with capillary expansion and contraction, and an amplitude term that represents the fluctuation intensity of the signal.
Building upon this line of research, our work explores a complementary perspective by incorporating frequency-domain modeling into diffusion-based rPPG estimation. 
Inspired by recent studies showing the advantages of frequency modeling for sequential signal analysis~\cite{yi2024filternet}, we introduce physiological frequency priors to guide the denoising process. 
While PhysDiff focuses on modeling temporal dynamics in the time domain, our approach explicitly captures the periodic structure of rPPG signals by leveraging frequency-domain representations, enabling more effective separation of cardiac rhythms from motion-induced interference.

\section{Preliminaries}\label{appendix_preliminaries}
\subsection{Multiplication of Complex Numbers}\label{complex_multiplication}
Consider two complex numbers $\mathcal{Z}_1 = a + jb$ and $\mathcal{Z}_2 = c + jd$, where $a$ and $c$ denote the real parts of $\mathcal{Z}_1$ and $\mathcal{Z}_2$, respectively, and $b$ and $d$ represent the corresponding imaginary parts. 
The multiplication of two complex numbers involves applying the distributive property of multiplication over addition, along with the identity $j^2 = -1$, where $j$ is the imaginary unit.
The product of $\mathcal{Z}_1$ and $\mathcal{Z}_2$ is computed as follows:
\begin{equation}
\begin{aligned}
    \mathcal{Z}_1 \mathcal{Z}_2 &= (a + jb)(c + jd)\\
    &= ac + a(jd) + jb(c) + jb(jd)\\
    &= ac + j(ad) + j(bc) + j^2(bd)\\
    &= (ac - bd) + j(ad + bc).
\end{aligned}
\end{equation}
Therefore, the multiplication of two complex numbers results in another complex number, whose real part is given by $ac-bd$ and whose imaginary part is given by $ad+bc$.

\subsection{Discrete Fourier Transform}\label{DFT}
The \emph{Discrete Fourier Transform} (DFT)~\cite{brigham1967fast} is a fundamental tool in signal processing and spectral analysis. 
It transforms a discrete-time signal from the temporal domain to the frequency domain, enabling a decomposition of the signal into its constituent frequency components. 
This facilitates the precise identification and analysis of underlying periodic patterns and oscillatory behavior.
Given a discrete real-valued temporal signal $\bm{x} \in \mathbb{R}^T$, its frequency-domain representation $\bm{x}^{F} \in \mathbb{C}^T$ is a complex-valued sequence defined by:
\begin{equation}
\begin{aligned}
    \bm{x}^{F}[i]
    &= \sum_{t=0}^{T-1} \bm{x}[t] \cdot e^{-j2\pi it/T} \\
    &= \underbrace{\sum_{t=0}^{T-1} \bm{x}[t] \cdot 
    \cos\!\left(\frac{2\pi it}{T}\right)}_{\text{Real Part}}
    - j\,\underbrace{\sum_{t=0}^{T-1} \bm{x}[t] \cdot 
    \sin\!\left(\frac{2\pi it}{T}\right)}_{\text{Imaginary Part}},
\end{aligned}
\end{equation}
where $i \in {0, 1, \ldots, T-1}$ indexes the discrete frequency bins, and $j$ is the imaginary unit such that $j^2 = -1$.
The corresponding physical frequency for the $i$-th bin is given by $\lambda_i = i f_s / T$ Hz, where $f_s$ is the sampling rate of the signal $\bm{x}$. 
For real-valued signals, the DFT exhibits conjugate symmetry:
\begin{equation}
    \bm{x}^{F}[T-n]=\overline{\bm{x}^{F}[n]}, \quad \text { for } n=1, \ldots,\lfloor T / 2\rfloor,
\end{equation}
allowing us to retain only the first $\lfloor T/2 \rfloor + 1$ frequency components without loss of information. Hence, we define the DFT operator as a mapping $\mathcal{F}: \mathbb{R}^{T} \rightarrow \mathbb{C}^{\lfloor T/2 \rfloor + 1}$ for computational efficiency.
Each complex coefficient $\bm{x}^{F}[i]$ in the frequency domain can be uniquely expressed in terms of its amplitude and phase:
\begin{equation}
\begin{aligned}
    A[i] = |\bm{x}^{F}[i]| 
    = \sqrt{ \operatorname{Re}(\bm{x}^{F}[i])^2 + \operatorname{Im}(\bm{x}^{F}[i])^2 }, \;\;\;\;
    \phi[i] = \tan^{-1}\!\left( 
    \frac{ \operatorname{Im}(\bm{x}^{F}[i]) }{ \operatorname{Re}(\bm{x}^{F}[i]) } 
    \right),
\end{aligned}
\end{equation}
where $\operatorname{Re}(\cdot)$ and $\operatorname{Im}(\cdot)$ denote the real and imaginary parts, respectively. 
The amplitude $A[i]$ reflects the energy concentration at frequency $\lambda_i$, while the phase $\phi[i]$ captures the temporal alignment of the sinusoidal component at that frequency.
Since the DFT is a bijective (invertible) transformation, the original time-domain signal $\bm{x}[t]$ can be exactly reconstructed from its frequency-domain representation $\bm{x}^{F}[i]$ via the Inverse Discrete Fourier Transform (IDFT), expressed as:
\begin{equation}
\begin{aligned}
\bm{x}[t] &= \mathcal{F}^{-1}(\bm{x}^{F})[t] 
=  \frac{1}{T}\sum_{i=0}^{T-1} \bm{x}^{F}[i] \cdot e^{j2\pi it/T}, \quad t = 0, 1, \ldots, T-1.
\end{aligned}
\end{equation}

\section{Theoretical Proof}\label{appendix_proof}
\subsection{Proof of Theorem 1}\label{theorem_1_proof}
\begin{theorem}
    (Frequency-domain Convolution Theorem)  The multiplication of two signals in the frequency domain is equivalent to the frequency transformation of a circular convolution of these two signals in the temporal domain, which can be summarized as:
    \begin{equation}
        \mathcal{F} [\mathbf{M}(v)\;\otimes\;\mathbf{Z}(v)] = \mathcal{F}(\mathbf{M}(v))\;\odot\;\mathcal{F}(\mathbf{Z}(v)),
    \end{equation}
    where $\otimes$ and $\odot$ represent the circular convolutional and element-multiplication operation, respectively, $\mathbf{M}(v)$ and $\mathbf{Z}(v)$ represent two signals for the time variable $v$, and $\mathcal{F}(\cdot)$ denotes the Discrete Fourier Transform.
\end{theorem}
\textit{Proof}. Let $\mathbf{M}(v)$ and $\mathbf{Z}(v)$ are two length $T$ signals.
Let $\mathbf{M}(v)$ and $\mathbf{Z}(v)$ be two discrete signals of length $T$, defined over the time index $v = 0, 1, \dots, T - 1$. Let their DFTs be denoted by $\mathcal{F}(\mathbf{M}(v))$ and $\mathcal{F}(\mathbf{Z}(v))$, respectively.
We define the circular convolution of $\mathbf{M}(v)$ and $\mathbf{Z}(v)$ as:
\begin{equation}
    \mathbf{Y}(v) = \mathbf{M}(v)\otimes\mathbf{Z}(v) = \sum_{u=0}^{T-1} \mathbf{M}(u) \cdot \mathbf{Z}((v-u) \mod T).
\end{equation}
The DFT of the resulting signal $\mathbf{Y}(v)$ is given by:
\begin{equation}
    \mathcal{F}(\mathbf{Y}(v)) = \sum_{v=0}^{T-1} \mathbf{Y}(v) \cdot e^{-j 2\pi i v/T}, \quad i = 0, 1, \dots, T - 1,
\end{equation}
where $j$ is the imaginary unit, and $i$ denotes the frequency index.
Substituting the expression for $\mathbf{Y}(v)$ into the DFT, we obtain:
\begin{equation}
\begin{aligned}
    \mathcal{F}(\mathbf{Y}(v))
    &= \sum_{v=0}^{T-1} \left( \sum_{u=0}^{T-1} \mathbf{M}(u) \cdot \mathbf{Z}((v-u)\mod T) \right) e^{-j 2\pi i v/T} \\
    &= \sum_{u=0}^{T-1} \mathbf{M}(u) \cdot \sum_{v=0}^{T-1} \mathbf{Z}((v-u) \mod T) \cdot e^{-j 2\pi i v/T}.
\end{aligned}
\end{equation}
Next, we perform a change of variable by letting $r = (v - u) \mod T$, which implies $v = (r + u) \mod T$:
\begin{equation}
\begin{aligned}
    \mathcal{F}(\mathbf{Y}(v))
    &= \sum_{u=0}^{T-1} \mathbf{M}(u) \cdot \sum_{r=0}^{T-1} \mathbf{Z}(r) \cdot e^{-\frac{j 2\pi i (r + u)}{T}} \\
    &= \sum_{u=0}^{T-1} \mathbf{M}(u) \cdot e^{-\frac{j 2\pi i u}{T}} \cdot \sum_{r=0}^{T-1} \mathbf{Z}(r) \cdot e^{-\frac{j 2\pi i r}{T}}.
\end{aligned}
\end{equation}
Rewriting the above expression, we have:
\begin{equation}
\begin{aligned}
    \mathcal{F}(\mathbf{Y}(v))
    &= \left( \sum_{u=0}^{T-1} \mathbf{M}(u) \cdot e^{-\frac{j 2\pi i u}{T}} \right)
    \cdot \left( \sum_{r=0}^{T-1} \mathbf{Z}(r) \cdot e^{-\frac{j 2\pi i r}{T}} \right) \\
    &= \mathcal{F}(\mathbf{M}(v)) \odot \mathcal{F}(\mathbf{Z}(v)).
\end{aligned}
\end{equation}
Thus, the Discrete Fourier Transform of the circular convolution of two signals $\mathbf{M}(v)$ and $\mathbf{Z}(v)$ is equivalent to the element-wise product of their respective DFTs, i.e., $\mathcal{F} [\mathbf{M}(v)\otimes\mathbf{Z}(v)] = \mathcal{F}(\mathbf{M}(v))\odot\mathcal{F}(\mathbf{Z}(v))$. 
This completes the proof.

\subsection{Proof of Proposition 1}\label{proposition_1_proof}
In the \textbf{Reverse Process} part of \textbf{Section 3.3}, we briefly stated that incorporating the physiological frequency prior $\mathbf{C}^{\mathbf{P}}$ can help reduce the uncertainty in the reverse diffusion process, thereby facilitating more accurate rPPG signal reconstruction. 
To formalize this intuition that physiological frequency priors provide additional information for the reverse diffusion transition, we present the following proposition.
\begin{proposition}\label{Proposition_1}
Let $\mathbf{Y}_{k}$ denote the noisy latent variable at diffusion step $k$, $\mathbf{X}$ denote the observed input (e.g., MSTmap features), and $\mathbf{C}^{\mathbf{P}}$ denote the physiological frequency prior. 
Conditioning the reverse diffusion process on $\mathbf{C}^{\mathbf{P}}$ reduces the uncertainty of the reverse transition:
\begin{equation}
\mathbf{H}(\mathbf{Y}_{k-1} \mid \mathbf{Y}_k, \mathbf{X}, \mathbf{C}^{\mathbf{P}})
<
\mathbf{H}(\mathbf{Y}_{k-1} \mid \mathbf{Y}_k, \mathbf{X}).
\end{equation}
\end{proposition}
\textit{Proof}. We use the notion of conditional entropy from information theory to quantify uncertainty. 
In the reverse process of DDPM~\cite{ho2020denoising}, the rPPG signal at step $k$, denoted $\mathbf{Y}_k$, is treated as a condition. 
The uncertainty of the reverse process can thus be expressed as:
\begin{equation}
\begin{aligned}
    \mathbf{H}\left( \mathbf{Y}_{k-1} \mid \mathbf{Y}_k \right)
    = -\int p_\theta\left( \mathbf{Y}_{k-1}, \mathbf{Y}_k \right) \log p_\theta\left( \mathbf{Y}_{k-1} \mid \mathbf{Y}_k \right)\mathrm{d} \mathbf{Y}_{k-1}.
\end{aligned}
\end{equation}
Similarly, in PhysDiff~\cite{qian2025physdiff}, the condition includes only the facial observation sequence $\mathbf{X}$, so the uncertainty is modeled as $\mathbf{H}( \mathbf{Y}_{k-1} \mid \mathbf{Y}_k, \mathbf{X} )$. 
In our proposed method, the condition is extended to include physiological frequency information $\mathbf{C}^{\mathbf{P}}$, yielding $\mathbf{H}( \mathbf{Y}_{k-1} \mid \mathbf{Y}_k, \mathbf{X}, \mathbf{C}^{\mathbf{P}} )$.
From the basic property of conditional entropy, we know:
\begin{equation}
    \mathbf{H}( \mathbf{Y}_{k-1} \mid \mathbf{Y}_k ) \leq \mathbf{H}( \mathbf{Y}_{k-1} ).
\end{equation}
Using the definition of mutual information, we have:
\begin{equation}
    \mathbf{I}( \mathbf{Y}_{k-1}; \mathbf{Y}_k ) = \mathbf{H}( \mathbf{Y}_{k-1} ) - \mathbf{H}( \mathbf{Y}_{k-1} \mid \mathbf{Y}_k ).
\end{equation}
According to Equation 11, we know that $\mathbf{I}( \mathbf{Y}_{k-1}; \mathbf{Y}_k ) > 0$, which implies:
\begin{equation}
    \mathbf{H}( \mathbf{Y}_{k-1} \mid \mathbf{Y}_k ) < \mathbf{H}( \mathbf{Y}_{k-1} ).
\end{equation}
Using the chain rule for entropy, we have:
\begin{equation}
    \mathbf{H}( \mathbf{Y}_{k-1}, \mathbf{Y}_k, \mathbf{X} ) = \mathbf{H}( \mathbf{Y}_{k-1} \mid \mathbf{Y}_k, \mathbf{X} ) + \mathbf{H}( \mathbf{Y}_k, \mathbf{X} ).
\end{equation}
Rewriting this using conditional entropy identities:
\begin{equation}
\begin{aligned}
    \mathbf{H}( \mathbf{Y}_{k-1} \mid \mathbf{Y}_k, \mathbf{X} )
    &= \mathbf{H}( \mathbf{X} \mid \mathbf{Y}_{k-1}, \mathbf{Y}_k )
    + \mathbf{H}( \mathbf{Y}_{k-1}, \mathbf{Y}_k ) 
    \quad - \mathbf{H}( \mathbf{X} \mid \mathbf{Y}_k )
    - \mathbf{H}( \mathbf{Y}_k ) \\
    &= \mathbf{H}( \mathbf{Y}_{k-1} \mid \mathbf{Y}_k )
    + \mathbf{H}( \mathbf{X} \mid \mathbf{Y}_{k-1}, \mathbf{Y}_k )
    - \mathbf{H}( \mathbf{X} \mid \mathbf{Y}_k ).
\end{aligned}
\end{equation}
From Equation 11, since $\mathbf{Y}_{k-1}$ contains one less step of noise compared to $\mathbf{Y}_k$, it is closer to the original observation. 
Therefore:
\begin{equation}
    \mathbf{H}( \mathbf{X} \mid \mathbf{Y}_{k-1}, \mathbf{Y}_k ) < \mathbf{H}( \mathbf{X} \mid \mathbf{Y}_k ).
\end{equation}
Substituting this inequality back gives:
\begin{equation}
    \mathbf{H}( \mathbf{Y}_{k-1} \mid \mathbf{Y}_k, \mathbf{X} ) < \mathbf{H}( \mathbf{Y}_{k-1} \mid \mathbf{Y}_k ).
\end{equation}
Following the same reasoning, we can conclude:
\begin{equation}
    \mathbf{H}( \mathbf{Y}_{k-1} \mid \mathbf{Y}_k, \mathbf{X}, \mathbf{C}^{\mathbf{P}} ) < \mathbf{H}( \mathbf{Y}_{k-1} \mid \mathbf{Y}_k, \mathbf{X} ).
\end{equation}
This result confirms that incorporating the physiological prior $\mathbf{C}^{\mathbf{P}}$ into the conditioning set reduces the entropy of the target distribution in the reverse process. 
Consequently, this reduction in uncertainty simplifies the learning task for the diffusion model, potentially leading to more efficient training and enhanced accuracy in rPPG signal estimation.
This completes the proof.

\section{Mathematical Derivation Details}\label{appendix_derivation}
\subsection{Derivations of Closed-form Forward Process}\label{detail_Y_k}
Assuming that the clean rPPG target distribution $q(\mathbf{Y}_0)$ is known, we can first sample a clean rPPG target as $\mathbf{Y}_0 \sim q(\mathbf{Y}_0)$.
According to the forward process defined in Equation 11, the noisy rPPG signal at step $k$ can be generated as:
\begin{equation}
    \mathbf{Y}_k = \sqrt{\alpha_k} \mathbf{Y}_{k-1} + \sqrt{1 - \alpha_k} \epsilon_k, \quad \epsilon_k \sim \mathcal{N}(\mathbf{0}, \mathbf{I}).
\end{equation}
Similarly, the previous step $\mathbf{Y}{k-1}$ can be expressed as:
\begin{equation}
    \mathbf{Y}_{k-1} = \sqrt{\alpha_{k-1}} \mathbf{Y}_{k-2} + \sqrt{1 - \alpha_{k-1}} \epsilon_{k-1}, \quad \epsilon_{k-1} \sim \mathcal{N}(\mathbf{0}, \mathbf{I}).
\end{equation}
By recursively substituting, we obtain:
\begin{equation}
\begin{aligned}
    \mathbf{Y}_k &= \sqrt{\alpha_k} \left( \sqrt{\alpha_{k-1}} \mathbf{Y}_{k-2} + \sqrt{1 - \alpha_{k-1}} \epsilon_{k-1} \right)  + \sqrt{1 - \alpha_k} \epsilon_k \\
    &= \sqrt{\alpha_k \alpha_{k-1}} \mathbf{Y}_{k-2} + \left( \sqrt{\alpha_k (1 - \alpha_{k-1})} \epsilon_{k-1} + \sqrt{1 - \alpha_k} \epsilon_k \right).
\end{aligned}
\end{equation}
Given that $\epsilon_{k-1}, \epsilon_k \sim \mathcal{N}(\mathbf{0}, \mathbf{I})$, the two noise terms are independent Gaussian variables. 
Therefore, their weighted sum is also Gaussian:
\begin{equation}
\begin{aligned}
    \sqrt{\alpha_k (1 - \alpha_{k-1})} \epsilon_{k-1} &\sim \mathcal{N}(\mathbf{0}, \alpha_k (1 - \alpha_{k-1}) \mathbf{I}), \\
    \sqrt{1 - \alpha_k} \epsilon_k &\sim \mathcal{N}(\mathbf{0}, (1 - \alpha_k) \mathbf{I}),
\end{aligned}
\end{equation}
and their sum follows:
\begin{equation}
\begin{aligned}
    \sqrt{\alpha_k (1 - \alpha_{k-1})} \epsilon_{k-1} + \sqrt{1 - \alpha_k} \epsilon_k \sim \mathcal{N}\left(\mathbf{0}, \left[\alpha_k (1 - \alpha_{k-1}) + (1 - \alpha_k)\right] \mathbf{I}\right).
\end{aligned}
\end{equation}
This implies that the overall expression can be rewritten in the same form as before:
\begin{equation}
    \mathbf{Y}_k = \sqrt{\alpha_k \alpha_{k-1}} \mathbf{Y}_{k-2} + \sqrt{1 - \alpha_k \alpha_{k-1}} \epsilon,
\end{equation}
where $\epsilon \sim \mathcal{N}(\mathbf{0}, \mathbf{I})$.
Continuing this recursive process, we eventually obtain:
\begin{equation}
\begin{aligned}
    \mathbf{Y}k &= \sqrt{\prod_{s=1}^k \alpha_s} \mathbf{Y}_0 + \sqrt{1 - \prod_{s=1}^k \alpha_s} \epsilon, \quad \epsilon \sim \mathcal{N}(\mathbf{0}, \mathbf{I}).
\end{aligned}
\end{equation}
This result corresponds to the closed-form expression of $\mathbf{Y}_k$ in the forward diffusion process, starting from a clean rPPG signal $\mathbf{Y}_0$.

\subsection{Derivation of Parameterized Reverse Process}\label{appendix:reverse_parameters}
We begin with Bayes' theorem to derive the reverse process:
\begin{equation}
\begin{aligned}
     &p_{\theta}(\mathbf{Y}_{k-1}|\mathbf{Y}_k, \mathbf{X}, \mathbf{C}^{\mathbf{P}})= p_{\theta}(\mathbf{Y}_{k}|\mathbf{Y}_{k-1}, \mathbf{X}, \mathbf{C}^{\mathbf{P}}) \frac{ p_{\theta}(\mathbf{Y}_{k-1}|\mathbf{X}, \mathbf{C}^{\mathbf{P}})}{ p_{\theta}(\mathbf{Y}_{k}|\mathbf{X}, \mathbf{C}^{\mathbf{P}})}.
\end{aligned}
\end{equation}
According to Equation 11, the expected $p_{\theta}(\mathbf{Y}_{k}|\mathbf{Y}_{k-1}, \mathbf{X}, \mathbf{C}^{\mathbf{P}})$ is:
\begin{equation}
    p_{\theta}(\mathbf{Y}_{k}|\mathbf{Y}_{k-1}, \mathbf{X}, \mathbf{C}^{\mathbf{P}}) \sim \mathcal{N}(\mathbf{Y}_k; \sqrt{\alpha_k} \mathbf{Y}_{k-1}, \beta_k \mathbf{I}).
\end{equation}
Furthermore, based on Equation 13, we also have:
\begin{equation}
\begin{aligned}
    p_{\theta}(\mathbf{Y}_{k-1}|\mathbf{X}, \mathbf{C}^{\mathbf{P}}) &\sim \mathcal{N}(\mathbf{Y}_{k-1}; \sqrt{\bar{\alpha}_{k-1}} \mathbf{Y}_0, (1 - \bar{\alpha}_{k-1}) \mathbf{I}),\\
    p_{\theta}(\mathbf{Y}_{k}|\mathbf{X}, \mathbf{C}^{\mathbf{P}}) &\sim \mathcal{N}(\mathbf{Y}_{k-1}; \sqrt{\bar{\alpha}_{k}} \mathbf{Y}_0, (1 - \bar{\alpha}_{k}) \mathbf{I}).
\end{aligned}
\end{equation}
Combining the above three Gaussian distributions, we can derive:
\begin{equation}
\begin{aligned}
    p_{\theta}(\mathbf{Y}_{k-1}|\mathbf{Y}_k, \mathbf{X}, \mathbf{C}^{\mathbf{P}}) \propto  \mathcal{N}(\mathbf{Y}_k; \sqrt{\alpha_k} \mathbf{Y}_{k-1}, (1-\alpha_k) \cdot \mathcal{N}(\mathbf{Y}_{k-1}; \sqrt{\bar{\alpha}_{k-1}} \mathbf{Y}_0, (1 - \bar{\alpha}_{k-1}) \mathbf{I}).
\end{aligned}
\end{equation}
Since the product of two Gaussians is also a Gaussian, we can compute the posterior distribution analytically using the standard Gaussian product rule. 
Specifically, the reverse process becomes:
\begin{equation}
    p_{\theta}(\mathbf{Y}_{k-1}|\mathbf{Y}_k, \mathbf{X}, \mathbf{C}^{\mathbf{P}}) = \mathcal{N}(\mathbf{Y}_{k-1};\mu_\theta(\mathbf{Y}_k, \mathbf{X}, \mathbf{C}^{\mathbf{P}},k), \Sigma_{\theta}),
\end{equation}
where the mean and covariance are given by:
\begin{equation}
\begin{aligned}
    \Sigma_{\theta} &= \left( \frac{1}{\beta_k} \mathbf{I} 
    + \frac{1}{1 - \bar{\alpha}_{k-1}} \mathbf{I} \right)^{-1} 
    = \frac{(1 - \bar{\alpha}_{k-1}) \beta_k}{1 - \bar{\alpha}_k} \mathbf{I}, \\[4pt]
    \mu_\theta(\mathbf{Y}_k, \mathbf{X}, \mathbf{C}^{\mathbf{P}}, k) &= \Sigma_{\theta} \left( 
        \frac{1}{\beta_k} \sqrt{\alpha_k} \mathbf{Y}_k 
        + \frac{1}{1 - \bar{\alpha}_{k-1}} \sqrt{\bar{\alpha}_{k-1}} \mathbf{Y}_0 
    \right) \\[-3pt]
    &= \frac{\sqrt{\alpha_k}(1 - \bar{\alpha}_{k-1})}{1 - \bar{\alpha}_k} \mathbf{Y}_k 
    + \frac{\sqrt{\bar{\alpha}_{k-1}} \beta_k}{1 - \bar{\alpha}_k} \mathbf{Y}_0.
\end{aligned}
\end{equation}
However, during inference, the clean signal $\mathbf{Y}_0$ is not directly accessible. 
Therefore, we train a neural network $f_\theta(\cdot)$ to predict an approximation $\hat{\mathbf{Y}}_0$ from the noisy input:
\begin{equation}
    \hat{\mathbf{Y}}_0 = f_\theta(\mathbf{Y}_k, \mathbf{X}, \mathbf{C}^{\mathbf{P}}, k),
\end{equation}
and substitute $\mathbf{Y}_0$ in the mean computation $\mu_\theta(\mathbf{Y}_k, \mathbf{X}, \mathbf{C}^{\mathbf{P}},k)$ with the predicted $\hat{\mathbf{Y}}_0$. 
This yields the parameterized reverse process:
\begin{equation}
    \mu_\theta(\mathbf{Y}_k, \mathbf{X}, \mathbf{C}^{\mathbf{P}}, k) = \frac{\sqrt{\alpha_k}(1 - \bar{\alpha}_{k-1})}{1 - \bar{\alpha}_k} \mathbf{Y}_k + \frac{\sqrt{\bar{\alpha}_{k-1}} \beta_k}{1 - \bar{\alpha}_k} \hat{\mathbf{Y}}_0,
\end{equation}
where $\hat{\mathbf{Y}}_0 = f_\theta(\mathbf{Y}_k, \mathbf{X}, \mathbf{C}^{\mathbf{P}}, k)$.
The variance is kept fixed as:
\begin{equation}
    \Sigma_\theta = \sigma_k^2 \mathbf{I}, \quad \text{with} \quad \sigma_k^2 = \frac{(1 - \bar{\alpha}_{k-1}) \beta_k}{1 - \bar{\alpha}_k}.
\end{equation}
Therefore, this approach enables efficient learning by directly predicting the clean rPPG signal $\mathbf{Y}_0$, thereby avoiding explicit noise estimation as in the original DDPM framework~\cite{ho2020denoising}.

\begin{table}[t]
\caption{\textbf{A list of key notations and their interpretation in this work.}}
\vspace{-2ex}
\label{tab:appendix_key_notations}
\centering
\renewcommand\arraystretch{1.0}
\tabcolsep 10pt
\resizebox{0.8\linewidth}{!}{
\begin{tabular}{c|l}
\toprule
\rowcolor[HTML]{f8f9fa} \textbf{Symbol} & \textbf{Description} \\
\midrule
$\mathbf{V}$ & Facial video sequence. \\
$T$ & Temporal dimension of the input video sequence. \\
$H$ & Spatial height of the input video. \\
$W$ & Spatial width of the input video. \\
$\mathbf{X}$ & MSTmap used as model input. \\
$N$ & Facial regions of interest (ROIs). \\
$C$ & Channel dimension of the MSTmap. \\
$\mathbf{Y}$ & Ground-truth rPPG signal. \\
$\mathbf{C}^{\mathbf{P}}$ & Frequency condition. \\
$D$ & $D$-dimensional latent space. \\
$\mathbf{Z}^{(l)}$ & Intermediate feature of $l$-th layer denoising network. \\
\textbf{CRL} & Cross-domain Representation Learning module.\\
\textbf{PFD} & Physiological Frequency Denoiser module.\\
\textbf{PBF} & Physiological Bandpass Filter module.\\
\textbf{PSM} & Physiological Spectrum Modulation module.\\
\textbf{ASS} & Adaptive Spectrum Selection module.\\
$\mathbb{R}$ & Real number field.\\
$\mathbb{C}$ & Complex number field.\\
$\mathcal{F}(\cdot)$ & Discrete Fourier Transform.\\
$\tau$ & Data‐driven threshold in the \textbf{ASS} module. \\
$\mathrm{CA}(\cdot,\cdot)$ & Cross-Attention.\\
$\mathbf{Y}_0$ & Clean rPPG signal.\\
$\hat{\mathbf{Y}}_0$ & Predicted rPPG signal.\\
\bottomrule
\end{tabular}}
\end{table}

\section{Model Details}\label{appendix_model}
\subsection{Key Notations and Their Interpretation}
As summarized in Tab.~\ref{tab:appendix_key_notations}, we provide a clear overview of the primary symbols and variables used throughout this work.

\begin{figure*}[t]
\centering
\includegraphics[width=1.0\linewidth]{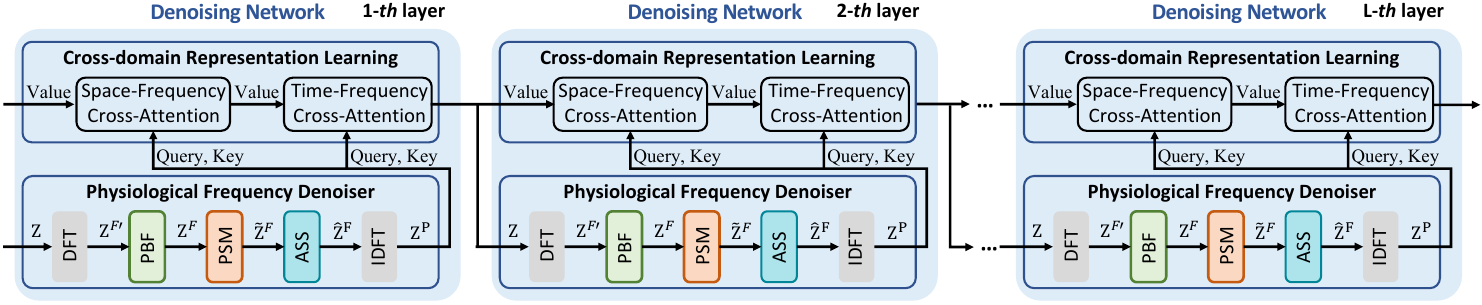}
\caption{Implementation details of the stacked denoising network in our \textcolor{SkyBlue}{\textbf{FreqPhys}}.}
\label{fig:detail_arc}
\end{figure*}

\subsection{Detailed Architecture of the Denoising Network}
As shown in Fig.~\ref{fig:detail_arc}, our denoising network is constructed by stacking $L$ sequential layers, each containing two major components: the \textit{Physiological Frequency Denoiser} (PFD) and the \textit{Cross-domain Representation Learning} (CRL) module.
The PFD serves as the first operation in each layer. Given an intermediate representation $\mathbf{Z}$, we first apply the Discrete Fourier transform (DFT) to obtain its frequency-domain representation. The transformed spectrum then undergoes three physiology-driven processes: the \textit{Physiological Bandpass Filter} (PBF) for removing out-of-band noise, the \textit{Physiological Spectrum Modulation} (PSM) for emphasizing cardiac-related harmonics, and the \textit{Adaptive Spectrum Selection} (ASS) for suppressing residual physiological interference. The refined spectrum is then converted back to the temporal domain via the inverse DFT, producing the enhanced feature $\mathbf{Z}^{\mathbf{P}}$.
Following the PFD, the CRL module integrates the cleaned physiological clues with spatial and temporal features. It consists of two cross-attention blocks: the \textit{Space-Frequency Cross-Attention}, which connects spatial features with frequency priors, and the \textit{Time-Frequency Cross-Attention}, which injects long-range temporal dependencies under frequency-aware guidance. This dual-path fusion enables the network to jointly exploit spectral regularities and temporal dynamics.
By repeating this structure over $L$ layers, the denoising network progressively removes noise and enhances physiological clues. Each layer first filters the signal with physiological priors and then enriches it via cross-domain attention, enabling iterative recovery of high-fidelity rPPG signals.

\subsection{Cross-domain Representation Learning}
\textbf{Space-Frequency Cross-Attention.} To capture spatial dependencies across facial ROIs guided by physiological frequency clues, we apply a multi-head cross-attention over the ROI dimension at each timestamp.
Assuming the inputs of $l$-th layer are the intermediate feature $\mathbf{Z}^{(l)} \in\mathbb{R}^{T\times N\times D}$ and physiological frequency representation $\mathbf{Z}^{\mathbf{P},(l)} \in\mathbb{R}^{T\times N\times D}$.
Then, for each timestamp $t$, the process of space-frequency interaction learning is formulated as:
\begin{equation}
    \begin{gathered}
        Q_{t}=\mathbf{Z}^{\mathbf{P},(l)}_{t}W^{Q}_{S}\in\mathbb{R}^{N\times D}, \;\; K_{t}=\mathbf{Z}^{\mathbf{P},(l)}_{t}W^{K}_{S}\in\mathbb{R}^{N\times D},  \;\; V_{t}=\mathbf{Z}^{(l)}_{t}W^{V}_{S}\in\mathbb{R}^{N\times D},\\
        \mathbf{Z}^{\mathbf{S},(l)}_{t} = \mathrm{LayerNorm}\left(\mathrm{Softmax}\left(\frac{Q_{t}K_{t}^{\top}}{\sqrt{D}} \right)V_{t}+\mathbf{Z}^{(l)}_{t})\right)\in\mathbb{R}^{N\times D}, \\
        \hat{\mathbf{Z}}^{\mathbf{S},(l)}_{t} = \mathrm{LayerNorm}\left(\mathbf{Z}^{\mathbf{S},(l)}_{t} + \mathrm{FeedForward}\left(\mathbf{Z}^{\mathbf{S},(l)}_{t} \right)\right)\in\mathbb{R}^{N\times D},
    \end{gathered}
\end{equation}
where $W^Q_S, W^K_S, W^V_S \in \mathbb{R}^{D \times D}$ are learnable projection matrices.
Finally, the outputs $\hat{\mathbf{Z}}^{\mathbf{S},(l)}$ at all timestamps are concatenated along the temporal dimension to obtain the updated intermediate feature:
\begin{equation}
    \mathbf{Z}^{\mathbf{S},(l)}\in\mathbb{R}^{T \times N\times D} \leftarrow \mathrm{Concat}\left(\{\hat{\mathbf{Z}}^{\mathbf{S},(l)}_{t}\}_{t=1}^T\right).
\end{equation}

\noindent\textbf{Time-Frequency Cross-Attention.} Complementary to space-frequency cross-attention, time-frequency cross-attention further models the temporal periodic dependencies within individual ROIs through frequency-guided cross-attention along the time axis.
For each ROI $n$, it can be formulated as:
\begin{equation}
    \begin{gathered}
        Q_{n}=\mathbf{Z}^{\mathbf{P},(l)}_{n}W^{Q}_{T}\in\mathbb{R}^{T\times D}, \;\;
        K_{n}=\mathbf{Z}^{\mathbf{P},(l)}_{n}W^{K}_{T}\in\mathbb{R}^{T\times D}, \;\;
        V_{n}=\mathbf{Z}^{\mathbf{S},(l)}_{n}W^{V}_{T}\in\mathbb{R}^{T\times D}, \\
        \mathbf{Z}^{\mathbf{T},(l)}_{n} =
        \mathrm{LayerNorm}\left(
        \mathrm{Softmax}\left(\frac{Q_{n}K_{n}^{\top}}{\sqrt{D}} \right)V_{n}
        +\mathbf{Z}^{\mathbf{S},(l)}_{n}
        \right)\in\mathbb{R}^{T\times D}, \\
        \hat{\mathbf{Z}}^{\mathbf{T},(l)}_{n} =
        \mathrm{LayerNorm}\left(
        \mathbf{Z}^{\mathbf{T},(l)}_{n} +
        \mathrm{FeedForward}\left(\mathbf{Z}^{\mathbf{T},(l)}_{n}\right)
        \right)\in\mathbb{R}^{T\times D},
    \end{gathered}
\end{equation}
where $W^Q_T, W^K_T, W^V_T \in \mathbb{R}^{D \times D}$ are independent learnable projection matrices.
Finally, the outputs $\hat{\mathbf{Z}}^{\mathbf{T},(l)}_{t}$ for all ROIs are concatenated along the spatial dimension to update the intermediate feature:
\begin{equation}
    \mathbf{Z}^{(l+1)}\in\mathbb{R}^{N\times T \times D} \leftarrow \mathrm{Concat}\left(\{\hat{\mathbf{Z}}^{\mathbf{T},(l)}_{n}\}_{n=1}^N\right).
\end{equation}

\section{Reproduction Details}\label{appendix_reproduction}
\subsection{Datasets Details} \label{details_datasets}
We conduct extensive experiments on six widely used datasets for remote physiological measurement. 
Representative samples from these datasets are illustrated in Fig.~\ref{fig:appendix_dataset}. 
Detailed descriptions of these datasets are as follows:
\begin{enumerate}
\item [(1)] {UBFC-rPPG}~\cite{bobbia2019unsupervised} is a small-scale yet widely used dataset consisting of 42 facial videos from 42 subjects. 
The videos are of high quality with minimal noise or motion artifacts, making UBFC-rPPG an ideal benchmark for evaluating model accuracy under relatively clean and controlled conditions.

\item [(2)] {PURE}~\cite{stricker2014non} is another small-scale dataset designed for testing under controlled motion conditions. 
It comprises 59 one-minute videos from 10 subjects, each participating in six scenarios: 1) sitting still, 2) talking, 3) slow head movement, 4) fast head movement, 5) small head rotation, and 6) medium head rotation.

\item [(3)] {BUAA}~\cite{xi2020image} is constructed to evaluate the robustness of rPPG models under different illumination conditions. 
The dataset contains facial video sequences captured under controlled lighting environments spanning a wide intensity range, from extremely dim scenes (below 10 lux) to normally illuminated settings. 
In our experiments, we follow~\cite{xie2025physllm,zhao2025phase} and only use videos recorded under illumination levels of at least 10 lux. 
Sequences captured in extremely low-light conditions are excluded, since severe image degradation in such scenarios typically requires dedicated enhancement techniques that fall outside the scope of this study.

\item[(4)] MMPD~\cite{tang2023mmpd} is a challenging medium-scale rPPG dataset collected entirely using smartphone cameras. 
It contains 660 one-minute facial videos from 33 subjects. 
To evaluate algorithm robustness under diverse conditions, the dataset includes four skin tone categories, four illumination settings (LED-high, LED-low, incandescent, and natural light), and multiple motion scenarios such as resting, head rotation, conversation, and walking. 
Additionally, MMPD introduces dedicated experiments to analyze motion effects in otherwise static scenes, where subjects first perform high-intensity activities (e.g., high knee lifts) to elevate heart rate before recording.

\begin{figure}[t]
\centering
\includegraphics[width=1.0\linewidth]{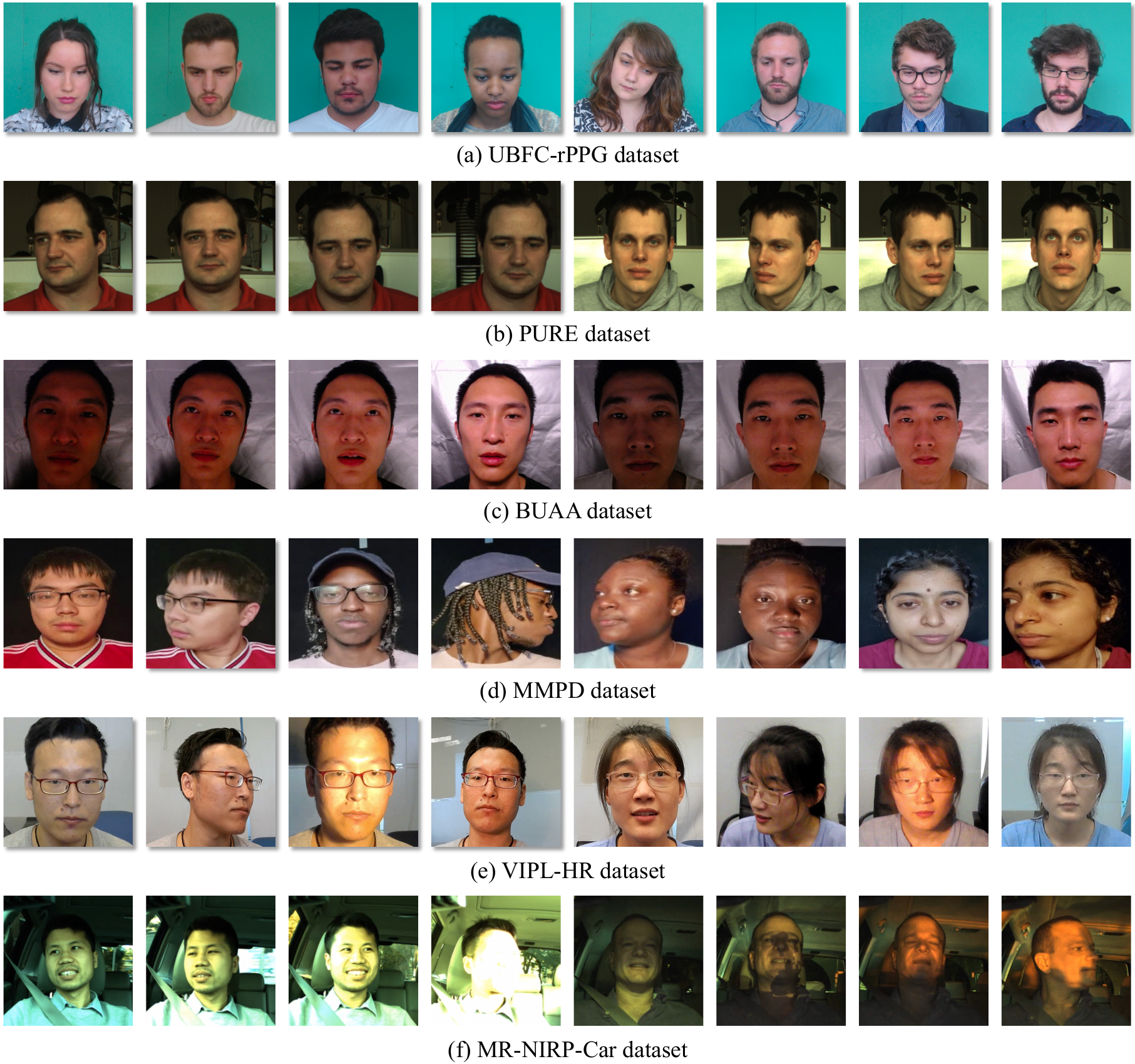}
\caption{Typical samples from six different remote physiological measurement datasets.} 
\label{fig:appendix_dataset}
\end{figure}

\item [(5)] {VIPL-HR}~\cite{niu2019rhythmnet} is a large-scale and highly challenging rPPG benchmark consisting of 2,378 facial RGB videos collected from 107 subjects. 
The videos are captured using three different devices, including a web camera, a smartphone front-facing camera, and an RGB-D camera. 
Furthermore, the dataset defines nine recording scenarios that combine diverse illumination conditions (e.g., bright and dark environments) with different levels of head motion (e.g., stable, talking, and large movements).  
Such variations in capture devices, lighting conditions, and motion patterns make VIPL-HR considerably more challenging and are widely used to evaluate the robustness and generalization ability of rPPG methods.

\item [(6)] {MR-NIRP-Car}~\cite{nowara2020near} is the first real-world rPPG benchmark collected in driving scenarios, designed to evaluate physiological measurement methods beyond controlled laboratory environments. 
The dataset contains 185 RGB videos recorded from 19 subjects while driving or sitting inside a parked car. 
During recording, subjects also performed natural activities such as speaking and random head movements, introducing realistic motion variations. 
Compared with laboratory datasets, MR-NIRP-Car presents substantially greater challenges due to extreme illumination fluctuations, day–night transitions, and diverse outdoor environmental conditions.
\end{enumerate}

Notably, \textbf{all dataset splits strictly follow the widely used benchmark protocols}~\cite{zou2025rhythmformer,zou2025rhythmmamba,qian2025physdiff,zhao2025phase,xie2025physllm}, \textbf{and all evaluations are conducted in a subject-independent manner to avoid identity leakage.} 
Specifically, for UBFC-rPPG, subjects 38 to 49 are used as the test set, while the remaining subjects are used for training. 
For PURE and BUAA, the datasets are sequentially divided by subject identity into training and testing sets with ratios of 6:4 and 7:3, respectively. 
For VIPL-HR, we adopt a subject-exclusive 5-fold cross-validation protocol. 
For MMPD and MR-NIRP-Car, the datasets are sequentially divided by subject identity into training, validation, and testing sets with a ratio of 7:1:2.

\subsection{Evaluation Metrics}\label{metric}
Following the common evaluation protocol in prior work~\cite{zou2025rhythmmamba,zou2025rhythmformer,qian2025physdiff,zhao2025phase,xie2025physllm}, we report four standard metrics for HR and HRV estimation: {mean absolute error (MAE)}, {root mean square error (RMSE)}, {Pearson’s correlation coefficient ($r$ or $\rho$)}, and {Standard Deviation (SD)}. 
The {MAE}, {RMSE}, and {SD} are measured in beats per minute (bpm), where lower values indicate better estimation accuracy, while a higher $r$ or $\rho$ reflects stronger correlation between predictions and ground truth.
Let $\mathbf{Y}_{pred}$ denote the predicted signal, $\mathbf{Y}_{gt}$ denote the ground truth signal, and $N$ be the total number of evaluation instances. The definitions of the adopted metrics are as follows:

\textbf{Mean Absolute Error (MAE):} It measures the average magnitude of the absolute differences between predicted and ground truth values, reflecting the overall prediction error without considering its direction.
\begin{equation}
    \mathbf{MAE} = \frac{1}{N} \sum_{n=1}^{N} \left| \mathbf{Y}^{n}_{gt} - \mathbf{Y}^{n}_{pred} \right|.
\label{eq:mae}
\end{equation}

\textbf{Root Mean Square Error (RMSE):} It evaluates the square root of the mean squared error, emphasizing larger errors due to the squaring operation, and is more sensitive to outliers than MAE.
\begin{equation}
    \mathbf{RMSE} = \sqrt{ \frac{1}{N} \sum_{n=1}^{N} \left( \mathbf{Y}^{n}_{gt} - \mathbf{Y}^{n}_{pred} \right)^2 }.
\label{eq:rmse}
\end{equation}

\textbf{Standard Deviation (SD):} It quantifies the dispersion of the prediction errors around their mean, providing insight into the consistency and reliability of the predictions.
\begin{equation}
    \mathbf{SD} = \sqrt{ \frac{1}{N} \sum_{n=1}^{N} \left( \mathbf{Y}^{n}_{e} - \overline{Y}_{e} \right)^2 },
\label{eq:sd}
\end{equation}
where the error term is defined as $\mathbf{Y}^{n}_{e} = \mathbf{Y}^{n}_{pred} - \mathbf{Y}^{n}_{gt}$, and $\overline{Y}_{e}$ denotes the mean error across all $N$ samples.

\textbf{Pearson’s Correlation Coefficient ($r$ or $\rho$):} It assesses the linear relationship between predicted and ground truth values, with higher values indicating stronger correlation and better temporal alignment.
\begin{equation}
    \bm{r(\rho)} = \frac{ \sum_{n=1}^{N} (\mathbf{Y}_{gt}^{n} - \overline{\mathbf{Y}}_{gt}) (\mathbf{Y}_{pred}^{n} - \overline{\mathbf{Y}}_{pred}) }{ \sqrt{ \sum_{n=1}^{N} (\mathbf{Y}_{gt}^{n} - \overline{\mathbf{Y}}_{gt})^2 \sum_{n=1}^{N} (\mathbf{Y}_{pred}^{n} - \overline{\mathbf{Y}}_{pred})^2 } },
\label{eq:pearson}
\end{equation}
where $\overline{\mathbf{Y}}_{gt}$ and $\overline{\mathbf{Y}}_{pred}$ are the sample means of the ground truth and predicted signals, respectively.

\section{Additional Experimental Results}\label{appendix_experiments}

\subsection{Domain Generalization Evaluation from Laboratory Scenarios to Real-world Scenarios}
We evaluate cross-domain generalization from laboratory datasets to the real-world driving dataset MR-NIRP-Car. Specifically, UBFC-rPPG (U), MMPD (M), and VIPL-HR (V) are used as source domains, while MR-NIRP-Car (C) serves as the unseen target domain. UBFC-rPPG is relatively stable and controlled, whereas MMPD and VIPL-HR contain more diverse motions and illumination conditions. MR-NIRP-Car represents the most challenging scenario with severe head movements and illumination fluctuations during driving.
As shown in Tab.~\ref{tab:cross-dataset_car}, traditional rPPG methods perform poorly in this setting, with average MAE values above 9 bpm. Deep learning methods improve performance but still suffer from significant degradation under domain shift. For example, EfficientPhys achieves an average RMSE of 12.87, while RhythmMamba reaches 14.48 due to unstable cross-domain behavior.

Interestingly, although MMPD and VIPL-HR contain substantially richer motion patterns and illumination variations than the relatively controlled UBFC-rPPG dataset, several deep learning methods do not benefit from this increased source-domain complexity. In fact, their cross-domain performance further deteriorates when trained on these more challenging datasets. For instance, RhythmMamba achieves an RMSE of 11.48 under U$\rightarrow$C but degrades dramatically to 18.92 under M$\rightarrow$C. EfficientPhys also increases from an RMSE of 11.09 to 13.00 under the same setting. This observation suggests that simply introducing more diverse source-domain data does not necessarily improve domain generalization, and many existing approaches struggle to learn domain-invariant physiological representations under large domain shifts.
In contrast, our \textcolor{SkyBlue}{\textbf{FreqPhys}} consistently achieves the best results across all settings. Specifically, it achieves RMSE values of 10.33, 9.62, and 9.22 under U$\rightarrow$C, M$\rightarrow$C, and V$\rightarrow$C, respectively, outperforming the SOTA EfficientPhys in every case. On average, \textcolor{SkyBlue}{\textbf{FreqPhys}} reduces RMSE from 12.87 to 9.72 (a 24.5\% improvement) while also improving the $r$ from 0.29 to 0.58. These results demonstrate that explicitly leveraging physiological frequency priors enables the model to better suppress motion-induced spectral distortions and illumination-related noise, leading to significantly improved generalization from controlled laboratory environments to complex real-world driving scenarios.

\begin{table*}[t]
\caption{Cross-domain generalization evaluation results on the MR-NIRP-Car dataset. U=UBFC-rPPG, M=MMPD, V=VIPL-HR, C=MR-NIRP-Car. \textbf{Bold}: best results.}
\vspace{-2.5ex}
\label{tab:cross-dataset_car}
\centering
\renewcommand\arraystretch{1.1}
\tabcolsep 1pt
\resizebox{1.0\linewidth}{!}{
\begin{tabular}{cc|ccc|ccc|ccc|ccc}
\toprule\toprule
\rowcolor[HTML]{f8f9fa} & & \multicolumn{3}{c|}{\textbf{U $\rightarrow$ C}} & \multicolumn{3}{c|}{\textbf{M $\rightarrow$ C}} & \multicolumn{3}{c|}{\textbf{V $\rightarrow$ C}} & \multicolumn{3}{c}{\textbf{Average}}\\ 
\rowcolor[HTML]{f8f9fa} \multirow{-2}{*}{\textbf{Method}} & \multirow{-2}{*}{\textbf{Venue}} & MAE$\downarrow$ & RMSE$\downarrow$ & $r\uparrow$ & MAE$\downarrow$ & RMSE$\downarrow$ & $r\uparrow$ & MAE$\downarrow$ & RMSE$\downarrow$ & $r\uparrow$ & MAE$\downarrow$ & RMSE$\downarrow$ & $r\uparrow$ \\ \midrule
\multicolumn{14}{l}{\emph{\textcolor{gray}{Traditional rPPG Methods}}}\\
GREEN~\cite{verkruysse2008remote} & Opt Express'08 & 11.01  & 14.60 & 0.19 & 11.01  & 14.60 & 0.19 & 11.01  & 14.60 & 0.19 & 11.01  & 14.60 & 0.19 \\
ICA~\cite{poh2010non} & Opt Express'10 & 10.07 & 13.89 & 0.28	& 10.07 & 13.89 & 0.28 & 10.07 & 13.89 & 0.28	& 10.07 & 13.89 & 0.28\\
CHROM~\cite{de2013robust} & TBE'13 & 9.70 & 13.71 & 0.28 & 9.70 & 13.71 & 0.28 & 9.70 & 13.71 & 0.28 & 9.70 & 13.71 & 0.28\\ 
POS~\cite{wang2016algorithmic} & TBE'16 & 9.32 & 13.56 & 0.21 & 9.32 & 13.56 & 0.21 & 9.32 & 13.56 & 0.21 & 9.32 & 13.56 & 0.21\\ 
\midrule
\multicolumn{14}{l}{\emph{\textcolor{gray}{Deep Learning-based rPPG Methods}}}\\
PhysNet~\cite{yu2019remote1} & BMVC'19 & 14.51 & 20.63 & 0.16 & 13.63 & 16.86 & 0.13 & 10.37 & 15.05 & 0.16 & 12.84 & 17.51 & 0.15 \\
PhysFormer~\cite{yu2022physformer} & CVPR'22 & 10.95 & 13.51 & 0.02 & 11.29 & 13.81 & 0.10 & 12.32 & 14.26 & 0.01 & 11.52 & 13.86 & 0.04\\ 
EfficientPhys~\cite{liu2023efficientphys} & WACV'23 & 6.80 & 11.09 & 0.48 & 7.69 & 13.00 & 0.39 & 11.69 & 14.53 & 0.01 & 8.73 & 12.87 & 0.29\\
RhythmMamba~\cite{zou2025rhythmmamba} & AAAI'25 & 6.71 & 11.48 & 0.48 & 15.99 & 18.92 & -0.21 & 10.05 & 13.03 & 0.14 & 10.92 & 14.48 & 0.14\\
\rowcolor[HTML]{EBFEE8} \textcolor{SkyBlue}{\textbf{FreqPhys}} \textbf{(Ours)} & - & {\textbf{6.69}} & {\textbf{10.33}} & {\textbf{0.53}} & {\textbf{6.46}} & {\textbf{9.62}} & {\textbf{0.59}} & {\textbf{5.97}} & {\textbf{9.22}} & {\textbf{0.62}} & {\textbf{6.37}} & {\textbf{9.72}} & {\textbf{0.58}}\\
\bottomrule\bottomrule
\end{tabular}}
\end{table*}

\subsection{Dual-source Domain Generalization Evaluation}
Following~\cite{zhao2025phase}, we further evaluate domain generalization under the dual-source setting, where models are trained on two datasets and tested on an unseen target domain. Specifically, we consider BUAA and MMPD as target datasets. BUAA contains videos recorded under a wide range of illumination intensities, while MMPD includes complex motion patterns, diverse lighting conditions, and large subject variations, making both datasets challenging benchmarks for cross-domain evaluation.

As shown in Tab.~\ref{tab:dual_DG_buaa}, several deep learning methods exhibit unstable cross-domain results when evaluated on the BUAA dataset. Interestingly, some methods even perform worse than traditional signal-processing methods under this challenging illumination setting. For example, PhysNet and PhysFormer obtain RMSE values of 21.48 and 22.17 under the P+U$\rightarrow$B setting, which are significantly worse than classical methods such as POS (RMSE 7.12) and CHROM (RMSE 8.29). Even strong deep learning baselines such as RhythmFormer achieve an average RMSE of 8.15. In contrast, the recent PHASE-Net reduces the average RMSE to 5.21. Our \textcolor{SkyBlue}{\textbf{FreqPhys}} further achieves the best performance across all settings, with RMSE values of 2.81, 2.58, and 2.77 under P+U$\rightarrow$B, P+M$\rightarrow$B, and U+M$\rightarrow$B, respectively. On average, our method reduces RMSE from 5.21 to 2.72 (a 47.8\% improvement) compared with PHASE-Net.
When evaluated on the more challenging MMPD dataset, as shoun in Tab.~\ref{tab:dual_DG_mmpd}), our method again achieves the best performance. Specifically, \textcolor{SkyBlue}{\textbf{FreqPhys}} achieves RMSE values of 14.91, 15.60, and 14.61 under P+U$\rightarrow$M, P+B$\rightarrow$M, and U+B$\rightarrow$M, respectively, outperforming the SOTA PHASE-Net (16.07, 15.96, and 17.47). These results demonstrate that explicitly modeling physiological frequency priors enables more robust domain-invariant representations, leading to improved generalization under complex illumination and motion variations.

\begin{table*}[t]
\caption{Cross-domain generalization evaluation results on the BUAA dataset, with dual-source training and single-target testing. U=UBFC-rPPG, P=PURE, B=BUAA, and M=MMPD. \textbf{Bold}: best results.}
\vspace{-2.5ex}
\label{tab:dual_DG_buaa}
\centering
\renewcommand\arraystretch{1.1}
\tabcolsep 1pt
\resizebox{1.0\linewidth}{!}{
\begin{tabular}{cc|ccc|ccc|ccc|ccc}
\toprule\toprule
\rowcolor[HTML]{f8f9fa} & & \multicolumn{3}{c|}{\textbf{P+U $\rightarrow$ B}} & \multicolumn{3}{c|}{\textbf{P+M $\rightarrow$ B}} & \multicolumn{3}{c|}{\textbf{U+M $\rightarrow$ B}} & \multicolumn{3}{c}{\textbf{Average}}\\ 
\rowcolor[HTML]{f8f9fa} \multirow{-2}{*}{\textbf{Method}} & \multirow{-2}{*}{\textbf{Venue}} & MAE$\downarrow$ & RMSE$\downarrow$ & $r\uparrow$ & MAE$\downarrow$ & RMSE$\downarrow$ & $r\uparrow$ & MAE$\downarrow$ & RMSE$\downarrow$ & $r\uparrow$ & MAE$\downarrow$ & RMSE$\downarrow$ & $r\uparrow$ \\ \midrule
\multicolumn{14}{l}{\emph{\textcolor{gray}{Traditional rPPG Methods}}}\\
GREEN~\cite{verkruysse2008remote} & Opt Express'08 & 6.89  & 10.39 & 0.60 & 6.89  & 10.39 & 0.60 & 6.89  & 10.39 & 0.60 & 6.89  & 10.39 & 0.60\\
CHROM~\cite{de2013robust} & TBE'13 & 6.09 & 8.29 & 0.51 & 6.09 & 8.29 & 0.51 & 6.09 & 8.29 & 0.51 & 6.09 & 8.29 & 0.51\\ 
POS~\cite{wang2016algorithmic} & TBE'16 & 5.04 & 7.12 & 0.63 & 5.04 & 7.12 & 0.63 & 5.04 & 7.12 & 0.63 & 5.04 & 7.12 & 0.63\\ 
\midrule
\multicolumn{14}{l}{\emph{\textcolor{gray}{Deep Learning-based rPPG Methods}}}\\
PhysNet~\cite{yu2019remote1} & BMVC'19 & 15.34 & 21.48 & -0.29 & 20.97 & 24.75 & 0.01 & 11.40 & 16.72 & 0.14 & 15.90 & 20.98 & -0.05 \\
PhysFormer~\cite{yu2022physformer} & CVPR'22 & 18.23 & 22.17 & 0.07 & 14.86 & 18.26 & 0.03 & 10.87 & 16.20 & 0.08 & 14.65 & 18.88 & 0.06\\ 
EfficientPhys~\cite{liu2023efficientphys} & WACV'23 & 4.60 & 8.06 & 0.72 & 4.15 & 7.14 & 0.77 & 3.00 & 5.18 & 0.89 & 3.92 & 6.79 & 0.79\\
RhythmFormer~\cite{zou2025rhythmformer} & PR'25 & 3.90 & 6.51 & 0.82 & 4.32 & 6.70 & 0.82 & 6.20 & 11.23 & 0.49 & 4.81 & 8.15 & 0.71\\
PHASE-Net~\cite{zhao2025phase} & CVPR'26 & 2.91 & 4.23 & 0.92 & 4.03 & 6.21 & 0.85 & 3.51 & 5.18 & 0.89 & 3.48 & 5.21 & 0.89\\
\rowcolor[HTML]{EBFEE8} \textcolor{SkyBlue}{\textbf{FreqPhys}} \textbf{(Ours)} & - & {\textbf{2.26}} & {\textbf{2.81}} & {\textbf{0.99}} & {\textbf{2.06}} & {\textbf{2.58}} & {\textbf{0.97}} & {\textbf{2.12}} & {\textbf{2.77}} & {\textbf{0.96}} & {\textbf{2.15}} & {\textbf{2.72}} & {\textbf{0.97}}\\
\bottomrule\bottomrule
\end{tabular}}
\end{table*}

\begin{table*}[t]
\caption{Cross-domain generalization evaluation results on the MMPD dataset, with dual-source training and single-target testing. U=UBFC-rPPG, P=PURE, B=BUAA, and M=MMPD. \textbf{Bold}: best results.}
\vspace{-2.5ex}
\label{tab:dual_DG_mmpd}
\centering
\renewcommand\arraystretch{1.1}
\tabcolsep 1pt
\resizebox{1.0\linewidth}{!}{
\begin{tabular}{cc|ccc|ccc|ccc|ccc}
\toprule\toprule
\rowcolor[HTML]{f8f9fa} & & \multicolumn{3}{c|}{\textbf{P+U $\rightarrow$ M}} & \multicolumn{3}{c|}{\textbf{P+B $\rightarrow$ M}} & \multicolumn{3}{c|}{\textbf{U+B $\rightarrow$ M}} & \multicolumn{3}{c}{\textbf{Average}}\\ 
\rowcolor[HTML]{f8f9fa} \multirow{-2}{*}{\textbf{Method}} & \multirow{-2}{*}{\textbf{Venue}} & MAE$\downarrow$ & RMSE$\downarrow$ & $r\uparrow$ & MAE$\downarrow$ & RMSE$\downarrow$ & $r\uparrow$ & MAE$\downarrow$ & RMSE$\downarrow$ & $r\uparrow$ & MAE$\downarrow$ & RMSE$\downarrow$ & $r\uparrow$ \\ \midrule
\multicolumn{14}{l}{\emph{\textcolor{gray}{Traditional rPPG Methods}}}\\
GREEN~\cite{verkruysse2008remote} & Opt Express'08 & 21.68  & 27.69 & -0.01 & 21.68  & 27.69 & -0.01 & 21.68  & 27.69 & -0.01 & 21.68  & 27.69 & -0.01\\
CHROM~\cite{de2013robust} & TBE'13 & 13.66 & 18.76 & 0.08 &13.66	& 18.76 & 0.08 & 13.66 & 18.76 & 0.08 &13.66	& 18.76 & 0.08\\ 
POS~\cite{wang2016algorithmic} & TBE'16 & 12.36 & 17.71 & 0.18 & 12.36 & 17.71 & 0.18 & 12.36 & 17.71 & 0.18 & 12.36 & 17.71 & 0.18\\ 
\midrule
\multicolumn{14}{l}{\emph{\textcolor{gray}{Deep Learning-based rPPG Methods}}}\\
PhysNet~\cite{yu2019remote1} & BMVC'19 & 11.00 & 17.30 & 0.28 & 13.20 & 16.70 & 0.23 & 13.50 & 17.00 & 0.09 & 12.57 & 17.00 & 0.20 \\
PhysFormer~\cite{yu2022physformer} & CVPR'22 & 11.40 & 17.50 & 0.23 & 13.90 & 18.60 & 0.21 & 13.20 & 16.50 & 0.12 & 12.83 & 17.53 & 0.19\\ 
EfficientPhys~\cite{liu2023efficientphys} & WACV'23 & 11.80 & 18.90 & 0.22 & 11.90 & 18.50 & 0.21 & 15.50 & 20.80 & 0.03 & 13.07 & 19.40 & 0.15\\
RhythmFormer~\cite{zou2025rhythmformer} & PR'25 & 10.50 & 16.72 & 0.28 & 13.98 & 19.46 & 0.12 & 12.57 & 17.45 & 0.15 & 12.35 & 17.88 & 0.18\\
PHASE-Net~\cite{zhao2025phase} & CVPR'26 & 9.76 & 16.07 & 0.39 & 11.38 & 15.96 & 0.30 & 11.84 & 17.47 & 0.15 & 10.99 & 16.50 & 0.28\\
\rowcolor[HTML]{EBFEE8} \textcolor{SkyBlue}{\textbf{FreqPhys}} \textbf{(Ours)} & - & {\textbf{9.27}} & {\textbf{14.91}} & {\textbf{0.50}} & {\textbf{11.30}} & {\textbf{15.60}} & {\textbf{0.38}} & {\textbf{10.20}} & {\textbf{14.61}} & {\textbf{0.47}} & {\textbf{10.26}} & {\textbf{15.04}} & {\textbf{0.45}}\\
\bottomrule\bottomrule
\end{tabular}}
\end{table*}

\begin{table*}[t]
\caption{Cross-domain generalization evaluation results on the UBFC-rPPG/PURE/BUAA/MMPD datasets, with three-source training and single-target testing. U=UBFC-rPPG, P=PURE, B=BUAA, and M=MMPD. \textbf{Bold}: best results.}
\vspace{-2.5ex}
\label{tab:three_DG}
\centering
\renewcommand\arraystretch{1.1}
\tabcolsep 1pt
\resizebox{1.0\linewidth}{!}{
\begin{tabular}{cc|ccc|ccc|ccc|ccc}
\toprule\toprule
\rowcolor[HTML]{f8f9fa} & & \multicolumn{3}{c|}{\textbf{P+B+M $\rightarrow$ U}} & \multicolumn{3}{c|}{\textbf{U+B+M $\rightarrow$ P}} & \multicolumn{3}{c|}{\textbf{P+U+M $\rightarrow$ B}} & \multicolumn{3}{c}{\textbf{P+U+B $\rightarrow$ M}}\\ 
\rowcolor[HTML]{f8f9fa} \multirow{-2}{*}{\textbf{Method}} & \multirow{-2}{*}{\textbf{Venue}} & MAE$\downarrow$ & RMSE$\downarrow$ & $r\uparrow$ & MAE$\downarrow$ & RMSE$\downarrow$ & $r\uparrow$ & MAE$\downarrow$ & RMSE$\downarrow$ & $r\uparrow$ & MAE$\downarrow$ & RMSE$\downarrow$ & $r\uparrow$ \\ \midrule
\multicolumn{14}{l}{\emph{\textcolor{gray}{Traditional rPPG Methods}}}\\
GREEN~\cite{verkruysse2008remote} & Opt Express'08 & 19.73 & 31.00 & 0.37 & 10.09 & 23.85 & 0.34 & 6.89 & 10.39 & 0.60 & 21.68 & 27.69 & -0.01\\
CHROM~\cite{de2013robust} & TBE'13 & 7.23 & 8.92 & 0.51 & 9.79 & 12.76 & 0.37 & 6.09 & 8.29 & 0.51 & 13.66 & 18.76 & 0.08\\
POS~\cite{wang2016algorithmic} & TBE'16 & 7.35 & 8.04 & 0.49 & 9.82 & 13.44 & 0.34 & 5.04 & 7.12 & 0.63 & 12.36 & 17.71 & 0.18\\ 
\midrule
\multicolumn{14}{l}{\emph{\textcolor{gray}{Deep Learning-based rPPG Methods}}}\\
PhysNet~\cite{yu2019remote1} & BMVC'19 & 13.83 & 23.66 & 0.35 & 33.23 & 35.25 & -0.15 & 12.75 & 16.37 & 0.08 & 13.37 & 16.64 & 0.29 \\
PhysFormer~\cite{yu2022physformer} & CVPR'22 & 10.29 & 18.13 & 0.60 & 19.75 & 24.30 & 0.24 & 22.09 & 26.21 & 0.03 & 13.90 & 19.30 & 0.06\\ 
EfficientPhys~\cite{liu2023efficientphys} & WACV'23 & 12.87 & 18.80 & 0.19 & 7.15 & 15.04 & 0.23 & 32.30 & 34.00 & -0.03 & 12.87 & 18.80 & 0.19\\
RhythmFormer~\cite{zou2025rhythmformer} & PR'25 & 14.71 & 22.49 & 0.43 & 21.11 & 25.76 & 0.04 & 6.04 & 10.84 & 0.42 & 16.14 & 20.50 & -0.11\\
PHASE-Net~\cite{zhao2025phase} & CVPR'26 & 10.04 & 15.56 & 0.65 & 2.86 & 9.66 & 0.91 & 2.56 & 3.25 & 0.96 & 10.33 & 16.20 & 0.40\\
\rowcolor[HTML]{EBFEE8} \textcolor{SkyBlue}{\textbf{FreqPhys}} \textbf{(Ours)} & - & {\textbf{0.57}} & {\textbf{1.29}} & {\textbf{0.99}} & {\textbf{2.39}} & {\textbf{5.48}} & {\textbf{0.98}} & {\textbf{2.17}} & {\textbf{2.45}} & {\textbf{0.99}} & {\textbf{10.23}} & {\textbf{15.14}} & {\textbf{0.44}}\\
\bottomrule\bottomrule
\end{tabular}}
\end{table*}

\subsection{Three-source Domain Generalization Evaluation}
Following~\cite{zhao2025phase}, we further evaluate domain generalization under the three-source setting, where models are trained on three datasets and tested on the remaining unseen domain. The results are reported in Tab.~\ref{tab:three_DG}.
Overall, \textcolor{SkyBlue}{\textbf{FreqPhys}} achieves the best performance across most settings. For example, under {P+B+M$\rightarrow$U}, our method achieves an RMSE of 1.29, significantly outperforming the SOTA PHASE-Net (RMSE: 15.56). Similarly, under {U+B+M$\rightarrow$P} and {P+U+M$\rightarrow$B}, our method achieves RMSE of 5.48 and 2.45, respectively, again surpassing all competing approaches. These results demonstrate that incorporating physiological frequency priors helps learn more robust representations across heterogeneous source domains.

In addition, we observed an interesting result that increasing the number of source datasets does not necessarily improve cross-domain performance. In particular, under the {P+U+B$\rightarrow$M} setting, most deep learning methods perform worse than their counterparts in the dual-source setting (Tab.~\ref{tab:dual_DG_mmpd}). For example, RhythmFormer obtains an RMSE of 20.50 and PHASE-Net reaches 16.20, both worse than their results under dual-source training. Even our method shows a slight performance drop compared with the dual-source results.
This phenomenon indicates that naively aggregating more source datasets may introduce larger domain discrepancies and optimization conflicts, making it harder to learn consistent physiological representations across domains. Importantly, this issue is observed across nearly all methods, indicating that it is a common challenge in multi-source domain generalization. Addressing this issue remains an open problem and represents an interesting direction for future research.

\subsection{Impact of Diffusion}
To evaluate the contribution of the diffusion backbone, we compare the full model with a variant that removes the iterative denoising process while keeping all frequency priors unchanged. 
The results are reported in Tab.~\ref{tab:appendix_abalation_diffusion}. 
Introducing the diffusion process leads to consistent improvements on both datasets. 
For example, on VIPL-HR the RMSE decreases from 6.64 to 6.34 and the $r$ increases from 0.82 to 0.86. 
On the more challenging MR-NIRP-Car dataset, the RMSE is reduced from 6.51 to 6.08 while the $r$ improves from 0.52 to 0.54. 
These results indicate that the diffusion-based refinement helps improve signal reconstruction when combined with the proposed frequency priors.

\begin{table}[t]
\caption{{Impact of diffusion model}.}
\label{tab:appendix_abalation_diffusion}
\centering
\renewcommand\arraystretch{1.1}
\tabcolsep 2pt
\resizebox{0.6\linewidth}{!}{
\begin{tabular}{c|ccc|ccc}
\toprule
\rowcolor[HTML]{f8f9fa} \multirow{2}{*}{\textbf{Method}} & \multicolumn{3}{c|}{\textbf{VIPL-HR}} & \multicolumn{3}{c}{\textbf{MR-NIRP-Car}}\\ 
& MAE$\downarrow$ & RMSE$\downarrow$ & $r\uparrow$ & MAE$\downarrow$ & RMSE$\downarrow$ & $r\uparrow$\\
\midrule
w/o Diffusion & 3.98 & 6.64 & {0.82} & 5.88 & 6.51 & 0.52\\
\rowcolor[HTML]{EBFEE8} \textbf{with Diffusion (Ours)} & \textbf{3.79} & \textbf{6.34} & \textbf{0.86} & \textbf{5.75} & \textbf{6.08} & \textbf{0.54}\\
\bottomrule
\end{tabular}}
\end{table}

\begin{table}[t]
\caption{{Comparison of different heart rate ranges}. We provide the RMSE$\downarrow$ results on the VIPL-HR dataset.}
\vspace{-2.5ex}
\label{tab:appendix_hr_range}
\centering
\renewcommand\arraystretch{1.0}
\tabcolsep 8pt
\resizebox{0.6\linewidth}{!}{
\begin{tabular}{l|ccc}
\toprule
\rowcolor[HTML]{f8f9fa} \textbf{Method} & \textbf{[40, 60)} & \textbf{[60, 100)} & \textbf{[100, 180]}\\
\midrule
PhysDiff~\cite{qian2025physdiff} & 8.73 & 4.57 & 13.40 \\
\rowcolor[HTML]{EBFEE8} \textcolor{SkyBlue}{\textbf{FreqPhys}} \textbf{(Ours)} & \textbf{6.83} & \textbf{4.25} & \textbf{10.73}\\
\bottomrule
\end{tabular}}
\end{table}

\begin{table}[ht!]
\caption{{Ablation study on the number of cross-attention layers in the Cross-domain Representation Learning module on the VIPL-HR dataset}.}
\vspace{-2.5ex}
\label{tab:ablation_layer}
\centering
\renewcommand\arraystretch{1.0}
\tabcolsep 16pt
\resizebox{0.6\linewidth}{!}{
\begin{tabular}{c|ccc}
\toprule
\rowcolor[HTML]{f8f9fa} \textbf{Method} & \textbf{MAE$\downarrow$} & \textbf{RMSE$\downarrow$} & \textbf{$r\uparrow$}\\
\midrule
1 & 7.42 & 10.60 & 0.66\\
2 & 4.84 & 7.82 & 0.81\\
3 & 3.90 & 6.52 & \textbf{0.86}\\
\rowcolor[HTML]{EBFEE8} \textbf{4} & \textbf{3.79} & \textbf{6.34} & \textbf{0.86}\\
5 & 3.97 & 6.60 & 0.85\\
6 & 5.19 & 7.74 & 0.82\\
\bottomrule
\end{tabular}}
\end{table}

\subsection{Impact of Different Heart Rate Ranges}
The frequency band of [0.66, 3.0] Hz (corresponding to [40, 180] bpm) is a widely adopted physiological prior in rPPG research, as it encompasses the heart rate variability of nearly all human populations. 
To evaluate performance under these representative physiological conditions, we conduct stratified statistical analyses on the VIPL-HR dataset. 
As reported in Tab.~\ref{tab:appendix_hr_range}, our method consistently surpasses SOTA PhysDiff across all HR subsets, demonstrating its robustness and reliability under diverse physiological states.

\subsection{Impact of Cross-attention Layers}
For the layers of cross-attention, we conducted ablations to analyze the optimal empirical setting of the cross-attention layer number. As shown in Tab.~\ref{tab:ablation_layer}, performance improves with the layer number up to a certain point (4 layers), after which diminishing returns are observed along with increased inference cost. We thus adopt 4 layers as a balanced parameter setting.

\section{Limitations and Future Work}\label{appendix_limitations}
Despite the promising performance of \textcolor{SkyBlue}{\textbf{FreqPhys}}, several limitations remain. 
First, our method mainly exploits physiological frequency priors for signal reconstruction, while other physiological characteristics, such as waveform morphology or multi-harmonic structures, are not explicitly modeled. Incorporating richer physiological priors may further improve signal fidelity and interpretability. 
Second, cross-dataset generalization from controlled domains to unseen complex scenarios remains challenging. Although our method improves robustness to motion and illumination variations, noticeable performance degradation still occurs when the domain gap becomes large, indicating that learning domain-invariant physiological representations remains an important direction for future research. 
Third, we observe that increasing the number of source domains does not always improve performance in multi-source domain generalization. Under the three-source setting, many methods (including ours) perform worse than in the dual-source setting, indicating that naively combining multiple datasets may introduce conflicting domain distributions and optimization difficulties. 
Finally, the evaluation of current rPPG methods is limited by existing datasets, which all only include healthy adults with heart rates in the normal range. As a result, model performance on special populations (e.g., infants, athletes, or patients with bradycardia, tachycardia, and arrhythmias) and extreme heart rate conditions remains largely unexplored. Building larger and more diverse datasets covering broader physiological conditions will be important for advancing robust and clinically reliable remote physiological measurement.

%
%
\end{document}